\documentclass[10pt,twocolumn,letterpaper]{article}

\usepackage{iccv}
\usepackage{times}
\usepackage{epsfig}
\usepackage{graphicx}
\usepackage{amsmath}
\usepackage{amssymb}

\usepackage{booktabs}
\usepackage{amsthm}
\theoremstyle{definition}
\newtheorem{definition}{Definition}[section]

\DeclareMathOperator*{\aff}{aff}
\newtheorem{theorem}{Theorem}
\newtheorem{lemma}[theorem]{Lemma}

\newcommand{\comm}[1]{\textcolor{blue}{#1}}
\usepackage[ruled,linesnumbered,vlined]{algorithm2e}
\usepackage{subcaption}

\iccvfinalcopy

\usepackage[pagebackref=true,breaklinks=true,letterpaper=true,colorlinks,bookmarks=false]{hyperref}

\def\iccvPaperID{****} 

\ificcvfinal\fi

\begin{document}

\title{Reachability Analysis of Convolutional Neural Networks}


\author{Xiaodong Yang, Tomoya Yamaguchi, Hoang-Dung Tran \\ Bardh Hoxha, Taylor T Johnson and Danil Prokhorov}

\maketitle
\ificcvfinal\thispagestyle{empty}\fi

\begin{abstract}
Deep convolutional neural networks have been widely employed as an effective technique to handle complex and practical problems. However, one of the fundamental problems is the lack of formal methods to analyze their behavior. To address this challenge, we propose an approach to compute the exact reachable sets of a network given an input domain, where the reachable set is represented by the face lattice structure. Besides the computation of reachable sets, our approach is also capable of backtracking to the input domain given an output reachable set. Therefore, a full analysis of a network's behavior can be realized. In addition, an approach for fast analysis is also introduced, which conducts fast computation of reachable sets by considering selected sensitive  neurons in each layer. The exact pixel-level reachability analysis method is evaluated on a CNN for the CIFAR10 dataset and compared to related works. The fast analysis method is evaluated over a CNN CIFAR10 dataset and VGG16 architecture for the ImageNet dataset. 

\end{abstract}


\section{Introduction}
%
Applications of Convolutional Neural Networks (CNNs) to safety-critical systems with learning-enabled components such as autonomous vehicles~\cite{kocic2019end,tian2018deeptest} and medical devices~\cite{turan2018deep,shvets2018automatic} have increased substantially.  However, neural networks are vulnerable to adversarial attacks~\cite{carlini2017towards,42503}. It has been shown that CNN performance is often susceptible to pixel-level variations in images, and a variety of methods have been developed to interpret the classifier prediction and understand the pixel-level fragility~\cite{modas2019sparsefool,su2019one,su2019attacking}. Many methods~\cite{huang2020one,zintgraf2017visualizing,bach2015pixel,montavon2018methods,samek2016evaluating} were proposed to interpret networks within the pixel level. To alleviate this problem, we take inspiration from a widely used approach in formal methods called reachability analysis. 
In the simplest form, reachability analysis refers to the process of computing the reachable state of a system given some input range and initial conditions. 
For image classification, reachability analysis may be used to find the reachable output sets with respect to an input domain. 
If reachable sets spans over multiple classifications, then the network may not be safe.

Recently, there has been significant research on the assessment of robustness of neural networks for safety verification. Most focus is directed toward feed-forward neural networks (FNN). They are mainly based on \textit{reachability}~\cite{gehr2018ai2,xiang2018output,tran2019star, yang2020reachability, xiang2017reachable}, \textit{optimization}~\cite{bastani2016measuring, lomuscio2017approach, raghunathan2018certified, dvijotham2018dual, tjeng2018evaluating, wong2018provable}, and \textit{search}~\cite{wang2018formal,huang2017safety,weng2018towards,katz2017reluplex, ehlers2017formal, bunel2018unified, dutta2018output}. Compared to FNNs, CNNs have more complex architectures with convolutional layers and max pooling layers. Various types of layers and deep architecture make the development of efficient and robust verification methods a challenging problem. Several approaches have been proposed for verification of CNNs~\cite{anderson2019optimization,katz2019marabou, kouvaros2018formal, ijcai2019-824, singh2018fast, singh2019abstract,singh2018boosting,tran2020cav, NEURIPS2020_f6c2a0c4}. However, few can handle real-world networks such as VGGNets~\cite{Simonyan15}. 

In this work, we propose a reachability analysis method that computes reachable sets of CNNs w.r.t. an input image domain. 
The framework enables practitioners to generate formal guarantees on the robustness of the image classifier with respect to pixel-level perturbations. 
We use the face lattice structure to compute and propagate the reachable set for each layer of the network until an output set is generated. 
In addition to the computation of reachable sets, our method also supports backtracking to the input domain given a subset of the reachable set. 
This is useful for debugging, since it provides information on the pixels and input ranges which are causing misclassification. 
Since formal verification of neural networks is an NP-Complete problem~\cite{katz2017reluplex}, we also provide a parametric, sound, under-approximation technique which efficiently approximates the reachable set of a CNN. 
This method can be tuned using a relaxation parameter in order to balance between computation speed and completeness. 
This technique is then utilized to present a falsification method for finding adversarial attacks. 
We demonstrate our approach on a CNN for the CIFAR10 dataset and a VGG16 network for the ImageNet dataset. 
We compare against state-of-the-art methods~\cite{singh2018fast,singh2018boosting,singh2019abstract,singh2019beyond,tran2020nnv} and demonstrate improvement in completeness and computation time.


The main contributions of the paper are as follows. First, we present a reachable set formulation using the notion of a face lattice. In order to conduct reachability analysis, several operations over the face lattice structure are presented. These are used to develop a method for reachability analysis of CNNs, which include ReLU and max-pooling layers. In addition, a backtracking method is presented for finding \textit{problematic} sets in the input range that are responsible for misclassification. Second, a fast under-approximation method for the reachable set is presented which allows the practitioner to select the trade-off between computation speed and completeness. Third, our reachability analysis method can be used to evaluate the robustness of networks trained with different defense methods, which is demonstrated in the experiments.



\section{Preliminaries}
The architecture of CNNs usually consist of multiple layers including convolutional layers ($L_c$), max pooling ($L_m$), batch-normalization ($L_b$), ReLU activation ($L_r$) and feed-forward layers ($L_l$). 
Let $\mathcal{N}$ denote a CNN with $n$ layers and $L_i$ be its $i$th layer where $i\in [1,2,\dots,n]$. Then, a CNN can be formulated as:
\begin{equation}
  \mathcal{N} = L_1 \otimes L_2 \otimes \dots \otimes L_n
\end{equation}
where $L_i \in \{L_c, L_m, L_b, L_r, L_l\}$, and $\otimes$ denotes the connection between adjacent layers. Inputs to the network will be sequentially processed by each layer.
The reachability analysis of neural networks refers to the computation of output reachable sets for the network given an input set. A \textit{reachable set} refers to a domain where all the element points are reachable. This set is usually represented using linear constraints or a geometric objects such as convex polytopes~\cite{xiang2018output}, Star set~\cite{tran2019star}, Zonotope~\cite{tran2019star,singh2018fast} and Abstract domain~\cite{singh2019abstract}. 
The reachable-set computation of the network $\mathcal{N}$ can be computed as follows:
\begin{equation}
  \mathcal{R}_{\mathcal{N}} = \mathcal{R}_n(\mathcal{R}_{n-1}(\dots\mathcal{R}_1(s)))
  \label{equ:pre2}
\end{equation}
where $s_{out} = \mathcal{R}_i (s_{in})$ with the input sets $s_{in}$ and their output sets $s_{out}$. The reachable set of one layer is used as an input to the next layer.

Depending on the type of the layer, the function $\mathcal{R}(\cdot)$ has two different operations. First, an \textit{affine transformation} operation is applied in layers $L_c$, $L_b$ and $L_l$. This operation linearly transforms element points of an input set with weights and bias (convolutional layer) or the mean and standard deviation (batch normalisation layer). Second, the input-domain division operation of the $max$ function is applied in layers $L_r$ and $L_m$. The $max$ function exhibits an unique linearity over each domain. For instance, the ReLU activation function $y\text{=}max(0,x)$ has two different input ranges $x<0$ and $x\geq0$ over which the output of the $max$ function has different linearities over input $x$. This is a source on non-linearity in CNNs and makes the  set computation challenging.

Recently, various approaches have been proposed to handle such challenges and they can be classified in two main categories. One is an over-approximation based approach, such as Zonotope and Abstract domain, and the other one is an exact-analysis approach such as Polytope, Star set and ImageStar. The approximation methods normally apply a geometric object such as Zonotope to approximate the $max$ function, which enables the simple representation of the nonlinearity of the $max$ function.
This simplification results in very efficient algorithms, however, the over-approximation bound may be large, since the approximation error w.r.t the $max$ function is accumulated in an exponential manner. 
The exact-analysis methods separately consider the divided input domains of the $max$ function. 
These types of methods compute the exact reachable sets and can guarantee a complete and sound verification of networks. But the number of reachable sets increases exponentially w.r.t. the neurons, yielding an expensive computation process. Several set representations have been utilized by other works~\cite{xiang2018output,tran2019star,tran2020cav}, aiming to simplify this process. However, these approaches do not scale very well for verification of deep neural networks. Here we propose a novel representation of sets using the face lattice structure which can exhibit a significantly higher efficiency compared to the aforementioned methods. 

\section{Reachable Set Representation}
In this section, we introduce the notion of a face lattice~\cite{henk200416,grunbaum2013convex} and show how it can be used for constructing reachable sets. Its properties make it particularly efficient for the computation of reachable sets for CNNs. 

\subsection{Face Lattice Structure}
\label{section: fl}
The face lattice of a set is a complete combinatorial structure that contains all its faces and partially orders them by face containment. A 3-dimensional tetrahedron and its face lattice is shown in (a) and (b) of Fig.~\ref{fig:combine}. Its faces are organized in terms of their dimensions. The highest-dimensional face $3\textbf{-}face$ is the tetrahedron itself. there are a total of 15 faces described by blue blocks and their dimension ranges from 0 to 3. The downwards-path along the solid line from one face to one of its adjacent faces represents their containment, which is transitive. 
For instance, the $2\textbf{-}$face: {1} which is the $plane_{2,3,4}$ contains the $1\textbf{-}$face: {2} which is the ${edge}_{2,4}$.
The indices of the $0\textbf{-}$face correspond to the vertices. Given a set $S$, let $\mathcal{L}$ denote its face lattice and $V$ denote its vertices' values, then $S$ is represented by 
\begin{equation}
    S = \langle \mathcal{L}, V\rangle
    \label{equ:set}
\end{equation}

The related concepts of \textit{supporting hyperplane} and \textit{faces} of a set are defined as follows. Other geometric details on the face lattice representation can be found in~\cite{henk200416}.
\begin{definition}[Supporting Hyperplane]
    A hyperplane $a^{\top}x+b=0$ denoted by $\mathcal{H}$ is a \textit{supporting hyperplane} of a $d$-dimensional bounded set $S$ in $\mathbb{R}^d$ space, if one of its closed halfspaces, ${a}^{\top}{x}+b\le 0$ or ${a}^{\top}{x}+b\ge 0$ contains $S$ while $S$ has at least one boundary point on the $\mathcal{H}$.
    \label{def:shyperplane}
\end{definition}

\begin{definition}[Face of a set]
    A face of a set S is an intersection of S with a \textit{supporting hyperplane} $\mathcal{H}$. When the dimension of $\aff(S\cap \mathcal{H})$ is $k$, the face is denoted as $k\textbf{-}f$ or $k\textbf{-}face$. Each face is itself a set. 
    The function $\aff(*)$ indicates the smallest affine set containing $*$. 
    A $d$-dimensional set contains a set of $0\textbf{-}$faces, $1\textbf{-}$faces, $\dots$, $(d\text{-}1)\textbf{-}$ faces as well as itself. 
    \label{def:face}
\end{definition}


\subsection{Affine Transformation}

\textit{Affine transformation} operations are very common in computing reachable sets. One of the advantages of the face lattice structure is that affine transofrmations only change the values of the vertices while preserving the face lattice structure~\cite{henk200416,grunbaum2013convex}. 
An affine transofrmation to a set $S$ is applied as follows. Given a set $S=\langle \mathcal{L}, V\rangle$, and an \textit{affine transformation} with a weight matrix $W$ and a bias vector $b$, then the output set $S'=\langle \mathcal{L}', V'\rangle$ is computed by
\begin{equation}
    \mathcal{L}'=\mathcal{L}, \ \ \ V'=WV + b
    \label{equ:4}
\end{equation}


 We note that an \textit{affine transformation} may project a set into a higher-dimensional or lower-dimensional space. 
 Projection onto higher dimensional space normally happens in the convolutional layer. On the other hand, projection onto lower dimensional space commonly happens in the linear layer where neurons are fully connected. 
 In this case, the dimension of the set will be reduced to the dimensions of the target space accordingly.
 When the set is affine transformed to a lower dimension, its face lattice is preserved as described in Eq.~\ref{equ:4}. The faces whose dimensions are higher than the target do not influence the computation of reachable sets, but the face lattice including all faces together will be processed for future operations, such that the geometric information from the previous layer is maintained.
 
 
\subsection{Split Operation}
\label{section:splitting}

The split operation is applied when a set spans over multiple input domains. For instance, when dealing with the $max$ function, the set is split in two in order to capture the two different linearities. 
For a general $max$ function $y = max\{x_1, x_2, \dots, x_l\}$, we have that $y=x_k$ if and only if $\forall x_i, i\neq k$, $x_k-x_i \geq 0$. 
We note that the input domain over which $x_k$ is the maximum is characterized by a set of linear constraints, and that input domains specifying different $x$s as the maximum are adjacent and divided by hyperplanes. 
In exact reachability analysis, a set will need to be split when it spans such domains. 
This case can be determined by inspecting whether the set intersects with those hyperplanes. 
In practice, making this determination is the most common and challenging issue in the exact reachability analysis of DNNs. It's usually handled by linear-optimization or similar approaches~\cite{katz2019marabou,kouvaros2018formal,tran2020cav,singh2018boosting,singh2019abstract}, which are computational expensive due to the high dimension of the sets and the large amount of neurons. The face lattice structure does not require such optimization and is therefore much more efficient. In order to formally define the split process, we first provide the following definitions. 

\begin{definition}[Vertex Types]
    When splitting a set $S$ with a hyperplane $\mathcal{H}$, a vertex is named \textit{positive} and denoted as $v^{+}$ if it is located in the closed halfspace $\text{a}^{\top}\text{x}+b > 0$. A vertex is named \textit{negative} and denoted as $v^{-}$ if it is located in the closed halfspace $\text{a}^{\top}\text{x}+b<0$. A vertex is \textit{zero} and denoted as $v^{0}$ if it locates on the hyperplane $\text{a}^{\top}\text{x}+b=0$ 
    \label{def:vertex_type}
\end{definition}

\begin{definition}[Set Types]
    A set is named \textit{positive} and denoted as $S^{+}$ if it has no $v^{-}$s w.r.t. the hyperplane $\mathcal{H}$, and is named \textit{negative} and denoted as $S^{-}$ if it has no $v^{+}$s w.r.t. the hyperplane $\mathcal{H}$.
    \label{def:set_type}
\end{definition}

Given a $\mathcal{H}:\text{a}^{\top}\textbf{x}+b=0$ and a set $S$, the split operation starts by determining whether there is an intersection between $H$ and $S$.
As the face lattice contains vertices of the set, this determination can be simplified to finding the distribution of vertices on sides of $\mathcal{H}$. This is achieved by substituting the $\textbf{x}$ in $\text{a}^{\top}\textbf{x} + b$ with vertex values. In terms of the vertex's distribution, the vertex type is defined in Definition~\ref{def:vertex_type}. Based on these types, set types are also defined in Definition~\ref{def:set_type}. 
Overall, there are three cases to consider:
\begin{enumerate}
  \item[1)] The hyperplane $\mathcal{H}$ intersects with $S$, where $S$ has both $v^{+}$s and $v^{-}$s w.r.t. to $\mathcal{H}$. In this case, $S$ is split into two non-empty subsets $S_1^{+}$ and $S_2^{-}$.
  \item[2)] The hyperplane $\mathcal{H}$ doesn't intersect with $S$, where $S$ doesn't have any $v^{-}$s w.r.t. to $\mathcal{H}$. In this case, $S$ itself is \textit{positive} w.r.t. the $\mathcal{H}$. 
  \item[3)] The hyperplane $\mathcal{H}$ doesn't intersect with $S$, where $S$ doesn't have any $v^{+}$s w.r.t. to $\mathcal{H}$. In this case, $S$ itself is \textit{negative} w.r.t. the $\mathcal{H}$. 
\end{enumerate}


In the following, we focus on the first case as the last two cases can be easily processed. The split process in the first case includes four steps as illustrated in Fig.~\ref{fig:combine}. They are:~1) identification of faces in each dimension that intersects with $\mathcal{H}$,~2) derivation of the face lattice structure that consists of the new faces generated from the intersection of $\mathcal{H}$ with those faces obtained from Step 1,~3) splitting of the original face lattice into two sub-structures according to the types of its vertices,~4) merging of the structure generated from Step 2 respectively with the two sub-structures from Step 3 and forming the final subsets in 
face lattice $S_1^{+}$ and $S_2^{-}$.

\begin{figure}[ht]
 \includegraphics[scale = 0.45]{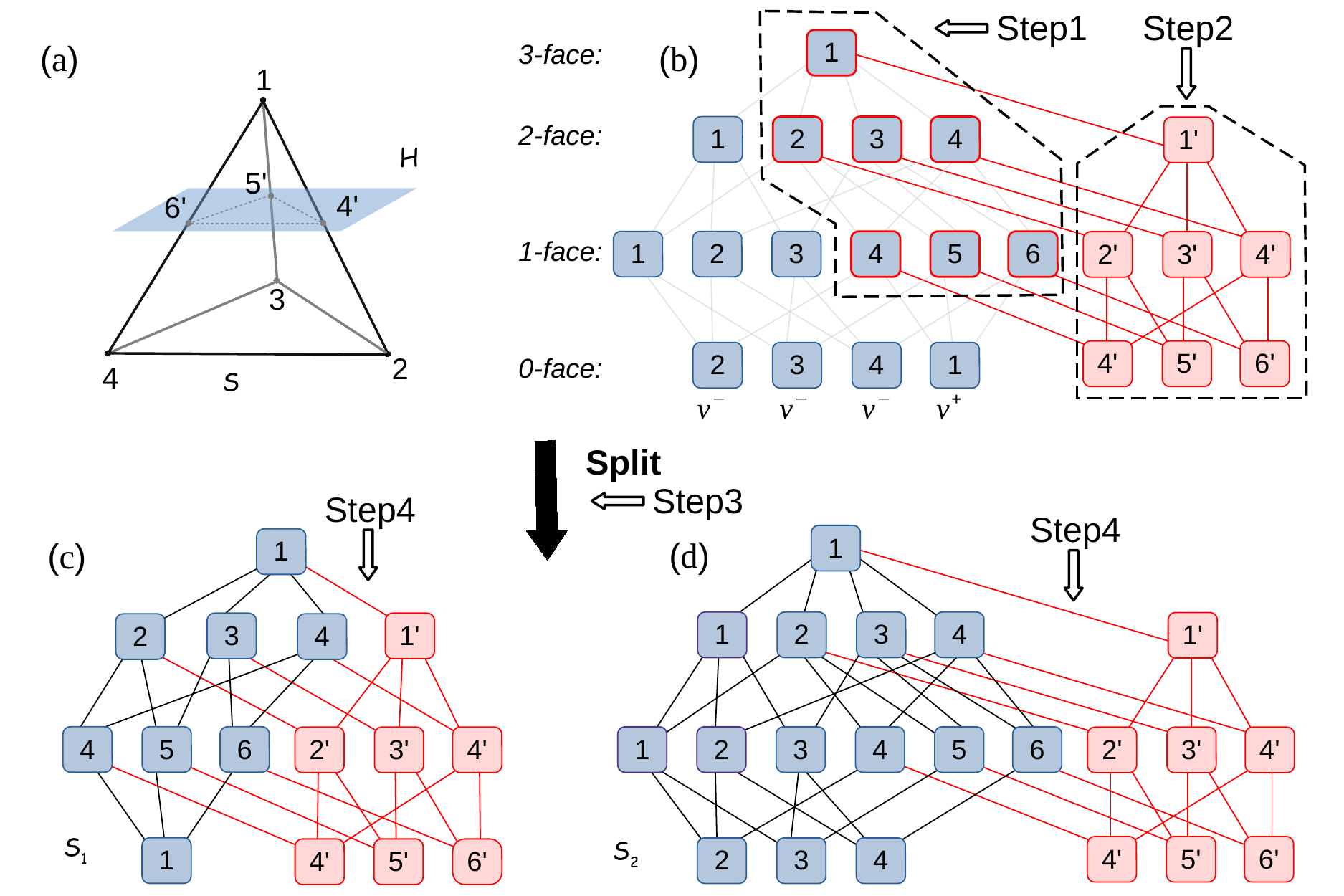}
 \caption{A demonstration of the split process of the face lattice of a tetrahedron $P$ with a hyperplane $\mathcal{H}$. In (b), the blocks filled with blue and connected by solid lines is the face lattice. The connection by solid lines represents face containment. The blocks with red frames and blue color are the faces intersecting with $\mathcal{H}$ in $S$, which are extracted from the listed $positive$ and $negative$ vertices in Step 2. The new face lattice generated in Step 2 is in red. In $(c)$ and $(d)$, the sub-face lattices consisting of the blue blocks are split from Step 3, and the whole structures with the red blocks represent the final face lattice. The one in (c) represents the \textit{positive} set and the one in (d) represents the \textit{negative} set.}
 \label{fig:combine}
\end{figure}

The identification of faces in Step 1 is based on Lemma \ref{lemma:intersect}, with which we can identify different dimensional faces that intersect with the hyperplane from the lower dimension to the higher along the containment relation. The identification starts with the $1\textbf{-}faces$ (edges). A $1\textbf{-}face$ is said to be intersecting with $\mathcal{H}$ if it contains a $v^{+}$ and a $v^{-}$. Based on these $1\textbf{-}faces$, we search faces along the upward path to the next upper dimension, and repeat this process until we reach the highest-dimensional faces. This ensures all faces are searched. This process is demonstrated in Fig.~\ref{fig:combine} (b). 


The derivation of the new face structure in Step 2 is based on Lemma~\ref{lemma:2}. The intersection of $\mathcal{H}$ with each face identified in Step 1 generates a new face. 
Additionally, their containment relation can also be inherited from the original faces based on Lemma~\ref{lemma:2}. This process is illustrated in Fig.~\ref{fig:combine} (b), where the red solid lines represent the inherited containment relations.
From another perspective, this derivation is essentially equivalent to a duplication of a sub-face lattice that only contains the faces that intersect with $\mathcal{H}$ in the original structure. 

The split operation in Step 3 separates the face lattice into two sub-face lattices according to the type of their vertices. 
Different dimensional faces belonging to one sub-face lattice are searched starting from the vertices along the containment relation up to the highest-dimensional faces. 
With $v^{+}\text{:}\{1\}$ and $v^{-}\text{:}\{2,3,4\}$, we can collect two sub-face lattices which only consists of blue blocks as shown in (c) and (d). In terms of containment relations inherited in Step 2, the sub-face lattices from Step 3 with the new face lattice from Step 2 are merged, generating the final two subsets $S_{1}^{+}$ and $S_{2}^{-}$, which is Step 4.


\begin{lemma}
Given a set $S$ in the face lattice structure and a hyperplane $\mathcal{H}$, if $\mathcal{H}$ intersects with a $k\textbf{-}face \subseteq S$, then the $\mathcal{H}$ also intersects with all $ (k\text{+}1)\textbf{-}faces$ where $ (k\text{+}1)\textbf{-}faces \subseteq S$ and $k\textbf{-}face \subseteq (k\text{+}1)\textbf{-}faces$.
\label{lemma:intersect}
\end{lemma}
\begin{proof}
Let $\mathcal{H} \cap k\textbf{-}face = \psi$ and $\psi \neq \emptyset$, then $\psi \subseteq \mathcal{H}$ and $\psi \subseteq k\textbf{-}face$. As $k\textbf{-}face \subseteq (k\text{+}1)\textbf{-}faces$, then $\psi \subseteq (k\text{+}1)\textbf{-}faces$ and $\psi \subseteq (\mathcal{H} \cap (k\text{+}1)\textbf{-}faces)$. Thus $\mathcal{H}$ intersects with $(k\text{+}1)\textbf{-}faces$.
\end{proof}

\begin{lemma}
Assume a hyperplane $\mathcal{H}$ intersect with a $(k\text{+}1)\textbf{-} face$ and $k\textbf{-} face$ in a set and generate new faces $k\textbf{-} face'$ and $(k\text{-}1)\textbf{-} face'$ accordingly. If $k\textbf{-} face \subseteq (k\text{+}1)\textbf{-} face$, then $(k\text{-}1)\textbf{-} face' \subseteq k\textbf{-} face'$.
\label{lemma:2}
\end{lemma}
\begin{proof}
As $k\textbf{-} face \subseteq (k\text{+}1)\textbf{-} face$, it can be inferred that $(k\textbf{-} face \cap \mathcal{H})\subseteq((k\text{+}1)\textbf{-} face \cap \mathcal{H})$, from which we can derive that $(k\text{-}1)\textbf{-} face' \subseteq k\textbf{-} face'$.
\end{proof}

\section{Computation of Reachable Sets of CNNs}
As discussed in Section 2, there are two primary operations on a set when it passes through a CNN. The \textit{affine transformation} are applied in the layers such as the convolutional layer and the linear layer, and the \textit{split} operation is applied to the ReLU layer and the max pooling layer. As introduced in Section~\ref{section: fl}, the \textit{affine transformation} only linearly updates the vertices of a set and its face lattice structure will be preserved. The algorithmic implementation of this operation is straightforward. Therefore, this section will mainly focus on the \textit{split} operation on sets in ReLU and max pooling layers. Additionally, a tracking method is also proposed which enables backtracking to the input domain of the network given an output reachable set. This method can be utilized to identify sets of adversarial examples.

\begin{figure*}[h]
  \centering
  \includegraphics[scale = 0.32]{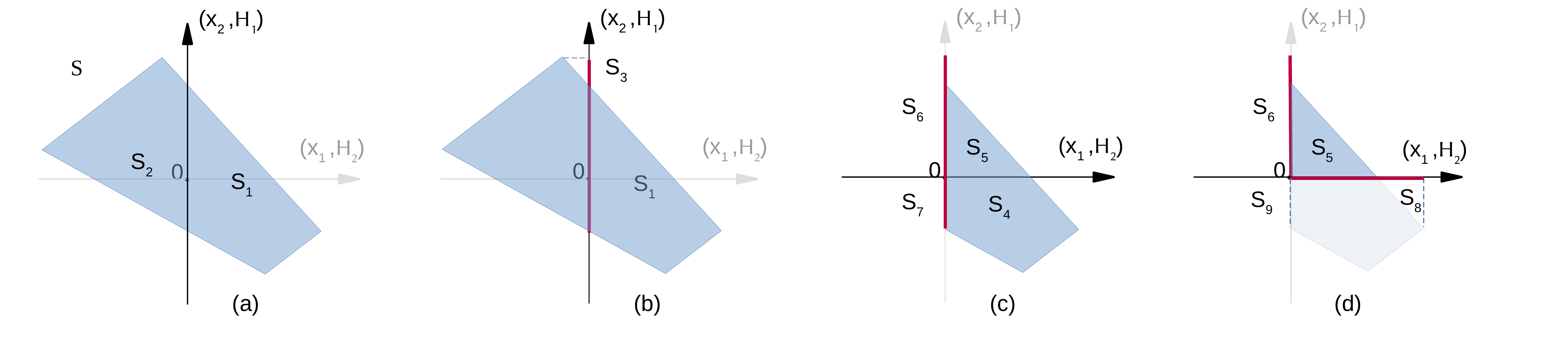}
  \caption{Demonstration of the \textit{split} operation in the ReLU layer}
  \label{fig:relu_split}
\end{figure*}

\subsection{ReLU Layer}

Assume a ReLU layer contains $n$ neurons and its input is denoted as $\textbf{x}\in \mathbb{R}^n$.
Then the output of the $k^{th}$ neuron is $y_k=\max\{0,x_k\}$ and its input space is divided by the boundary $x_k = 0$ into two domains with different linearity. The boundary can be generalized to the hyperplane $\mathcal{H}_k:\textbf{a}^{\top}\textbf{x}=0$ where $\textbf{a}=[a_1, a_2,\dots,a_n]$ and $a_k=1$ and $a_{i,i\neq k} =0$. When the input is a set $S$ where $\textbf{x}\in S$, it will be sequentially processed with the hyperplane $\mathcal{H}$ of each neuron. The sets generated w.r.t. one hyperplane will be further processed by the next one until all hyperplanes in this layer are considered.

The process w.r.t. one hyperplane $\mathcal{H}_k$ is as follows. In terms of the discussion in Section~\ref{section:splitting}, $positive$ and $negative$ sets can be obtained w.r.t. $\mathcal{H}_k$.  Depending on the linearity of the two domains of the $max$ function, the \textit{positive} set will stay unchanged as an output w.r.t to $\mathcal{H}_k$. On the other hand, the \textit{negative} set will be projected on $\mathcal{H}_k$ and transformed to a new set as an output, which is essentially done by setting the $x_k$ element of all the vertices in the set to 0.
Let the function $\mathcal{T}_{relu}^{[k]}$ denote this process, then the function $\mathcal{R}$ in Eq.~\eqref{equ:pre2} for the ReLU layer is equal to 
\begin{equation}
  \mathcal{R}(S) = \mathcal{T}_{relu}^{[n]}(\mathcal{T}_{relu}^{[n-1]}(\dots\mathcal{T}_{relu}^{[1]}(S)))
  \label{equ:relu_layer}
\end{equation}

{An example is used to illustrate this process as shown in Figure~\ref{fig:relu_split}. The ReLU layer contains two neurons $n_1$ and $n_2$. The input set $S \subset \mathbb{R}^2$ and element points $\textbf{x} =[x_1, x_2]^{\top}\in S$. When the set $S$ pass through the layer, it will be split sequentially by two hyperplanes that are imposed by two ReLU activation functions $max(0,x_1)$ and $max(0,x_2)$ from the $n_1$ and $n_2$. These two hyperplanes are respectively $\mathcal{H}_1: \{[1,0]^{\top}\textbf{x} = 0\}$ and $\mathcal{H}_2: \{[0, 1]^{\top}\textbf{x} = 0\}$. 
Together, there are four unique domains having different linearity over the input \textbf{x} to this layer, and they are
\begin{align*}
    & \{\textbf{x}~|~[1,0]^{\top}\textbf{x} \geq 0~and~ [0, 1]^{\top}\textbf{x} \geq 0\} \\
    & \{\textbf{x}~|~[1,0]^{\top}\textbf{x} \geq 0~and~ [0, 1]^{\top}\textbf{x} \leq 0\} \\
    & \{\textbf{x}~|~[1,0]^{\top}\textbf{x} \leq 0~and~ [0, 1]^{\top}\textbf{x} \geq 0\} \\
    & \{\textbf{x}~|~[1,0]^{\top}\textbf{x} \leq 0~and~ [0, 1]^{\top}\textbf{x} \leq 0\}
\end{align*}}

{The input set $S$ is first split by Hyperplane $\mathcal{H}_1$ into two subsets $\{S_1,S_2\}$.  The $S_2$  locates in the negative halfspace of $\mathcal{H}_1:[1,0]^{\top}\textbf{x} \leq 0$ and thus is the \textit{negative} set w.r.t. $\mathcal{H}_1$. Due to the property of the negative domain in the ReLU function, $S_2$ will be projected on the $\mathcal{H}_1$, generating a new set $S_3$. This process is essentially setting the $x_1$ of all the element points in the set to zeros. It's also equivalent to a linear transformation and the matrix is an identity matrix with first diagonal entry being zero. While the sets locating in the positive halfspace of $\mathcal{H}_1$ will remain. So far, sets $\{S_1, S_3\}$ are obtained from the \textit{split} by $\mathcal{H}_1$. Then, these sets will be split by $\mathcal{H}_2$ into four subsets $\{S_4, S_5, S_6, S_7\}$ as shown in (c), where $\{S_4, S_7\}$ are \textit{negative} sets w.r.t. $\mathcal{H}_2$. Similarly, sets $S_4$ and $S_7$ will be mapped on $\mathcal{H}_2$ by setting the $x_2$ of all their element points to zeros, generating new sets $S_8$ and $S_9$ which is the original point. Then the final outputs of this layer are $\{S_5, S_6, S_8, S_9\}$.
Let the function $\mathcal{T}_{relu}^{[i]}()$ denotes the \textit{split} process by a hyperplane $\mathcal{H}_i$ as well as the subsequent linear projection for sets that locate in the negative halfspace. Then we have the following relation:
\begin{align*}
  \{S_2,S_3\} &= \mathcal{T}_{relu}^{[1]}(\{S_1\})\\
  \{S_5,S_7,S_8,S_9\} &= \mathcal{T}_{relu}^{[2]}(\{S_2,S_3\})
\end{align*}
}

\SetKwProg{Fn}{Function}{}{end}
\SetKwFunction{Fnrelu}{FnReLULayer}
\SetKwFunction{Fnneurons}{FnValidNeurons}
\SetKwFunction{Fnsp}{FnSplit}
\SetKwFunction{Fnpop}{pop}
\SetKwFunction{Fnmap}{FnMap}
\SetKwFunction{Fnls}{list}
\SetKwFunction{Fnapp}{extend}


For one layer, given an input set, the maximum number of output reachable sets is $2^n$, where $n$ is the number of neurons. Our experimental results indicate that the actual number is associated with the input set's volume. An input set with a larger volume will generate a larger number of sets, as it has a higher likelihood of being split by hyperplanes. This process is described in Algorithm~\ref{alg:relu}. It's based on a recursive function which aims to reduce the repeated computation on the hyperplanes that don't intersect with sets. Given an input set $S$, it starts with \Fnrelu{S, neurons=empty, flag=False}. In each recursion, only one valid neuron whose hyperplane truly intersects with $S$ is considered. Function \Fnneurons is first to compute all such valid neurons denoted as \textit{news} as well as the neurons \textit{negs} where $S$ is \textit{negative} w.r.t. the target hyperplanes. Subsequently, $S$ will be projected on those hyperplanes using Function \Fnmap. Function \Fnsp denotes the split process introduced in Section~\ref{section:splitting}.



\begin{algorithm}
\KwResult{Output reachable sets }
\Fn(){\Fnrelu{S, neurons, flag=True}}{
\If(){neurons \textbf{is} empty \textbf{and} flag}{
    \Return{\Fnls{S}} 
}
\textit{news}, \textit{negs} = \Fnneurons{S, neurons}\;

\textit{S}.\Fnmap{\textit{negs}}\;

\If(){\textit{news} \textbf{is} empty}{
    \Return{\Fnls{S}}
}
\textit{aneuron} = \textit{news}.\Fnpop{0}\;

\textit{outputs} = \Fnsp{S, aneuron}\;
\textit{all} = \Fnls{}\;
\For{s \textbf{in} outputs}{
    \textit{all}.\Fnapp{\Fnrelu{s, news}}
}
\Return{all}
}
\caption{Computation in ReLU layers}
\label{alg:relu}
\end{algorithm}

\subsection{Max Pooling Layer}
The purpose of the max pooling layer is to reduce the dimensionality of an input by pooling with a $max$ function. 
For parameter configurations, we consider the most common kernel size $2\times2$ and stride size $2\times2$. Each pool contains four unique element dimensions and there are no overlaps between pools. Suppose a pool contains dimensions $\{x_1, x_2, x_3, x_4\}$, then the pooling process will be formulated by $y=\max\{x_1, x_2, x_3, x_4\}$. Element $x_i$ will be the output if $\forall x_{j(j\neq i)},\ x_i- x_j\geq 0$. We notice that the input domain over which $x_i$ is the maximum is the common \textit{positive} halfspace of three hyperplanes $\mathcal{H}_{(x_i, x_{j(j\neq i)})}\text{:}\ x_i-x_j = 0$. When an input set to this pool spans this domain, its subset can be computed by sequentially applying the split operation with the three aforementioned hyperplanes. Considering the other three dimensions as well, there are totally six unique hyperplanes and they are $\mathcal{H}_{(x_1, x_2)}, \mathcal{H}_{(x_1, x_3)},\mathcal{H}_{(x_1, x_4)},\mathcal{H}_{(x_2, x_3)}, \mathcal{H}_{(x_2, x_4),} $, and $ \mathcal{H}_{(x_3, x_4)}$, respectively. With them, the input space of the $max$ function is divided into four domains over which they exhibit unique linearity, and they are 
\begin{align*}
    &\{\textbf{x}~|~x_1-x_2\geq 0~\&~x_1-x_3\geq0~\&~x_1-x_4\geq0\}\\
    &\{\textbf{x}~|~x_2-x_1\geq 0~\&~x_2-x_3\geq0~\&~x_2-x_4\geq0\}\\
    &\{\textbf{x}~|~x_3-x_1\geq 0~\&~x_3-x_2\geq0~\&~x_3-x_4\geq0\}\\
    &\{\textbf{x}~|~x_4-x_1\geq 0~\&~x_4-x_2\geq0~\&~x_4-x_3\geq0\}
\end{align*}

The split operation is the same as introduced in Section~\ref{section:splitting}. However, in this case, the split subset will be handled differently. In the ReLU layer, the \textit{negative} sets will be mapped to their corresponding hyperplane. While in the max pooling layer, non-maximum dimensions will be eliminated. Here, the concepts of the \textit{negative} set and the \textit{positive} set are extended to be w.r.t a specific dimension as defined in Definition~\ref{def:set_type_extended}.
\begin{definition}[Set Types]
    In the max pooling layer, given a hyperplane $\mathcal{H}_{(x_i, x_j)}$, a set is named $x_i$-\textit{negative} or $x_j$-\textit{positive} if it belongs to the closed halfspace $x_i-x_j\leq 0$, and is named $x_j$-\textit{negative} or $x_i$-\textit{positive} if it belongs to the closed halfspace $x_j-x_i\leq 0$.
    \label{def:set_type_extended}
\end{definition}

When a set is split by a hyperplane $\mathcal{H}_{(x_i, x_j)}$ into two non-empty subsets, we have a $x_i$-\textit{negative} and a $x_j$-\textit{negative} subset. The dimension for which the subset is \textit{negative} will be eliminated from its vertices. Then these new sets will be processed w.r.t. the next hyperplane.
Let $\mathcal{T}_{max}^{[x_i, x_j]}$ denote the process with the hyperplane $\mathcal{H}^{[x_i, x_j]}$, and $\mathcal{P}$ denote the sequential splits of an input set $S$ by the hyperplanes. Then we can derive Eq.~\eqref{equ:pt}. 
\begin{multline}
\label{equ:pt}
    \mathcal{P} = \mathcal{T}_{max}^{[x_3, x_4]}(\mathcal{T}_{max}^{[x_2, x_4]}(\mathcal{T}_{max}^{[x_1, x_4]}( \\
    \mathcal{T}_{max}^{[x_2, x_3]}(\mathcal{T}_{max}^{[x_1, x_3]}(\mathcal{T}_{max}^{[x_1, x_2]}(\{S\}))))))
\end{multline}

{The process $\mathcal{T}_{max}^{[x_i, x_j]}$ is illustrated in Figure~\ref{fig:maxpooling}. There are two dimensions $x_1$ and $x_2$ and correspondingly there is one hyerplane $\mathcal{H}_{(x_1, x_2)}:x_1-x_2 = 0$. As shown in Figure (a), the input set $S$ is first split by $\mathcal{H}_{(x_1, x_2)}$ into two subsets $S_1$ and $S_2$ which respectively are $x_2$-\textit{negative} and $x_1$-\textit{negative}. The elimination of dimensions is shown in Figure (b), where the 2-dimensional sets $S_1$ and $S_2$ respectively generate new 1-dimensional sets $S_3$ and $S_4$.
Finally, this pooling yields reachable sets $\{S_3, S_4\}$. Let $\mathcal{T}_{max}^{[x_i, x_j]}$ be the split operation with the hyperplane $\mathcal{H}^{[x_i, x_j]}$ as well as the subsequent linear mapping, then we have
\begin{equation}
  \{S_3, S_4\} = \mathcal{T}_{max}^{[x_1, x_2]}(\{S\})
\end{equation}}

\begin{figure}[h]
  \centering
  \includegraphics[scale=0.34]{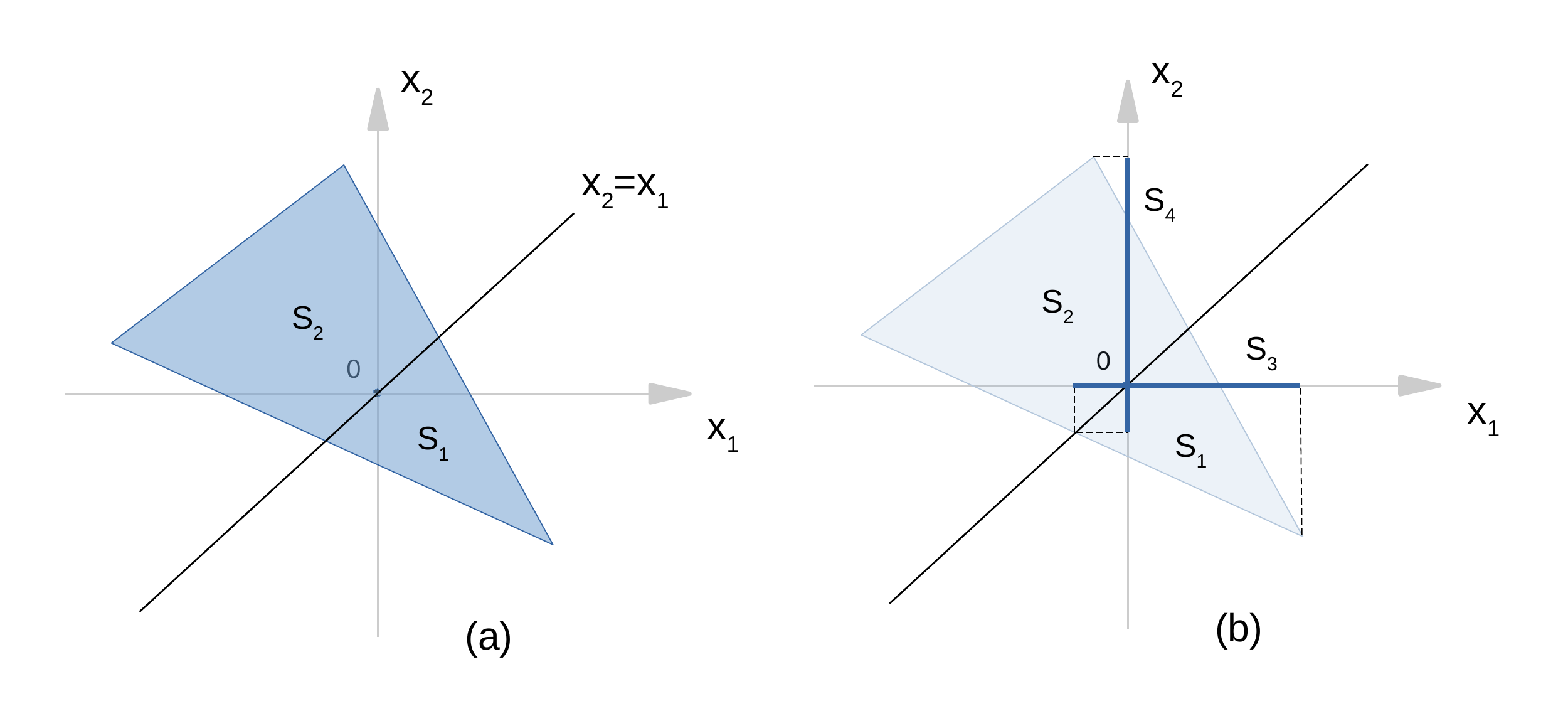}
  \caption{Demonstration of the split operation in the maxpooling layer}
  \label{fig:maxpooling}
\end{figure}

Suppose the layer has $n$ pooling operations and the $k^{th}$ pooling is denoted as $\mathcal{P}_k$, the $\mathcal{R}$ in Eq.~\eqref{equ:pre2} for the max pooling layer with an input set $S$ can be formulated as Eq.~\eqref{equ:rp}.
 
\begin{equation}
\label{equ:rp}
    \mathcal{R}(S) = \mathcal{P}_n(\mathcal{P}_{n-1}(\dots\mathcal{P}_1({S})))
\end{equation}


\SetKwProg{Fn}{Function}{}{end}
\SetKwFunction{Fnmp}{Fnmax pooling}%
\SetKwFunction{Fnelim}{FnEliminateDims}
\SetKwFunction{Fnpools}{FnValidPools}
\SetKwFunction{Fnpop}{pop}
\SetKwFunction{Fnhps}{FnValidHplanes}
\SetKwFunction{Fnapp}{extend}
\SetKwFunction{ls}{list}
\SetKwFunction{Fnsp}{FnSplit}

Given an input set, the maximum number of the output reachable sets of this layer is $2^{3n}$. Similar to the ReLU layer, the actual number is related to the input set's volume. In addition to Eq.~\eqref{equ:rp}, The process in the max pooling layer is also described in Algorithm~\ref{alg:max pooling}. It's conducted in a similar recursive manner as Algorithm~\ref{alg:relu}. Given an input set $S$, this algorithm starts with \Fnmp{S,\ pools=empty,\ flag=False}. Function \Fnpools helps identify a new set of pools \textit{news} where intersections exist between their hyperplanes and $S$. The set $S$ is processed with one pool in each recursion, which is described by Eq.~\eqref{equ:pt} and Line 9-14 in the algorithm. Function \Fnhps helps identify hyperplanes. Function \Fnsp splits a set with a hyperplane.
\begin{algorithm}[ht]
\KwResult{Output reachable sets }
\Fn(){\Fnmp{\textit{S}, \textit{pools}, \textit{flag}=True}}
{
\If{\textit{pools} \textbf{is} empty \textbf{and} \textit{flag}}{
    \Return{S. \Fnelim{}} \;
}
\textit{news} = \Fnpools{pools} \;
\If{\textit{news} \textbf{is} empty }{
    \Return{S. \Fnelim{}}\;
}
\textit{apool} = \textit{news}.\Fnpop{\text{0}} \;

\textit{hps} = \Fnhps{apool} \;
\textit{outputs} = \ls{S} \;
\For(){\textit{hp} \textbf{in} \textit{hps}}
{
    \textit{temp} = \ls{}\;
    \For{s \textbf{in} \textit{outputs}}{
     \textit{temp}.\Fnapp{\Fnsp{s, hp}}
     }
    \textit{outputs} = \textit{temp}\;
}
\textit{all} = \ls{} \;
\For{s \textbf{in} outputs}
{ 
    \textit{all}.\Fnapp{\Fnmp{s, news}} \;
}
\Return{\textit{all}}\;
}
\caption{Computation in max-pooling layers}
\label{alg:max pooling}
\end{algorithm}

\subsection{Backtracking Method}
Given a subset of the output reachable sets, this method enables backtracking to find the corresponding subset in the input domain. This enables us to identify adversarial input space. As described in Eq.~\eqref{equ:relu_layer} and \eqref{equ:pt}, the domains that exhibits different linearity in the $max$ function are always separately considered for processing reachable sets. Therefore, for each output reachable set, there always exists a subset in the input space over which the neural network is linear, and this subset is named \textit{linear region}. There has been recent work on the quantification of linear regions to asses the expressivity of neural networks~\cite{montufar2014number,serra2018bounding, hanin2019complexity}.

Our backtracking method tracks the \textit{linear region}s of sets when they are sequentially processed by neurons. Here we represent the \textit{linear region}s using vertices (V-representation). As sets are related to their \textit{linear regions} with a linear map, the vertices of the \textit{linear region} and the vertices of the set itself are in a one-to-one correspondence. This relation is maintained in the \textit{affine transformation} on the set. For the \textit{split} process in the ReLU neuron $\mathcal{T}_{relu}$ and the \textit{max pooling}  $\mathcal{T}_{max}$, in addition to splitting the set, we also split its \textit{linear region} accordingly, such that their one-to-one relation can be maintained and tracked. The \textit{linear region} of the input set to the network is the input set itself. Overall, all the computed reachable sets have their \textit{linear region}s and thus any violation in the reachable sets can be backtracked to the input space. 

\section{Fast Reachability Analysis}
\label{section: fast}

With the face lattice structure, we can efficiently compute the exact output reachable sets for most of the deep neural networks and guarantee complete and sound verification.  
However, due to the exponential computational complexity in the ReLu layer and max pooling layer, the analysis of a deeper CNN with larger input sets is still challenging. Here we propose an efficient alternative approach which only considers the neurons that are the most sensitive  to the output. 
The impact of each neuron is ranked using the output's gradient with respect to their input. 
We use a \textit{relaxation} factor to select a fraction of highest ranked neurons in each layer and compute reachability analysis using only the selected neurons. 
The \textit{relaxation} factor defines the trade-off between computation time and completeness of the computed reachable sets. Our experimental results show that even with a small \textit{relaxation} factor, the exact reachable sets can still be well approximated as shown in Fig.~\ref{fig:fastreach1} and \ref{fig:fastreach}.

\SetKwProg{Fn}{Function}{}{end}
\SetKwFunction{Fnnetwork}{FnMaxPooling}%
\SetKwFunction{Fnaffine}{FnAffineTrans}
\SetKwFunction{Fncd}{FnImpactNeurons}
\SetKwFunction{Fnrelu}{FnReLULayerFast}
\SetKwFunction{Fnmp}{FnMaxPoolingFast}
\SetKwFunction{Fnsp}{FnSplit}

\begin{algorithm}[ht]
\KwData{\textit{$S_{in}$} \comm{\# an input sets}; \textit{$\delta$} \comm{\# the \textit{relaxation} factor}
}
\KwResult{\textit{Sets} \comm{\# a list of output sets}
}
\textit{Sets} = \textit{$S_{in}$} \;
\For{layer \textbf{in} network\_layers}{
    \uIf{layer \textbf{is} affine-transformation layer}{
    \textit{Sets} = \Fnaffine{Sets} 
    }
    \uElseIf{layer \textbf{is} ReLu layer 
}{
    \textit{neurons} = \Fncd{$\delta$} 
    \;
    \textit{Sets} = \Fnrelu{Sets, neurons} 
    }
\uElseIf{layer \textbf{is} Maxpooling layer}{
    \textit{neurons} = \Fncd{$\delta$} \;
    \textit{Sets} = \Fnmp{Sets, neurons}
    }
}
\Return{Sets}
\caption{Faset Computation of Reachable sets}
\label{alg:fast-network}
\end{algorithm}

The reachable set computation is slightly different using the fast analysis method. In the ReLU layer, for the selected neurons, computation remains unchanged. However, for each of the non-selected neurons, at most two subsets can be generated by the split operation and only one out of two subsets, which has the highest number of vertices, is retained.
We chose the subset with more vertices since it is more likely that it carries more geometric information. 
In the max pooling layer, for the selected neurons, computation remains unchanged. Here, we apply a split operation for each of the non-selected neurons and discard the subset which is \textit{positive} w.r.t. a non-selected neuron (see Defn.~\ref{def:set_type_extended}). When two neurons are non-selected, we keep the subset with more vertices to avoid empty output. 


{The fast analysis is presented in Algorithm~\ref{alg:fast-network}. Functions \Fnrelu and \Fnmp accept a list of sets as input. Given a \textit{relaxation} factor, Function \Fncd selects \textit{neurons} on which the fast computation of reachable sets are conducted. Functions \Fnrelu and \Fnmp are the revised functions respectively from Algorithm~\ref{alg:relu} and \ref{alg:max pooling}. The revised part of these functions is that when two subset are generated from Function \Fnsp and the current neuron is not in \textit{neurons}, the subset with more vertices will be kept while the other one will be discarded. }

\begin{figure}[ht]
	\centering
	\begin{subfigure}{.25\textwidth}
		\includegraphics[width=\textwidth]{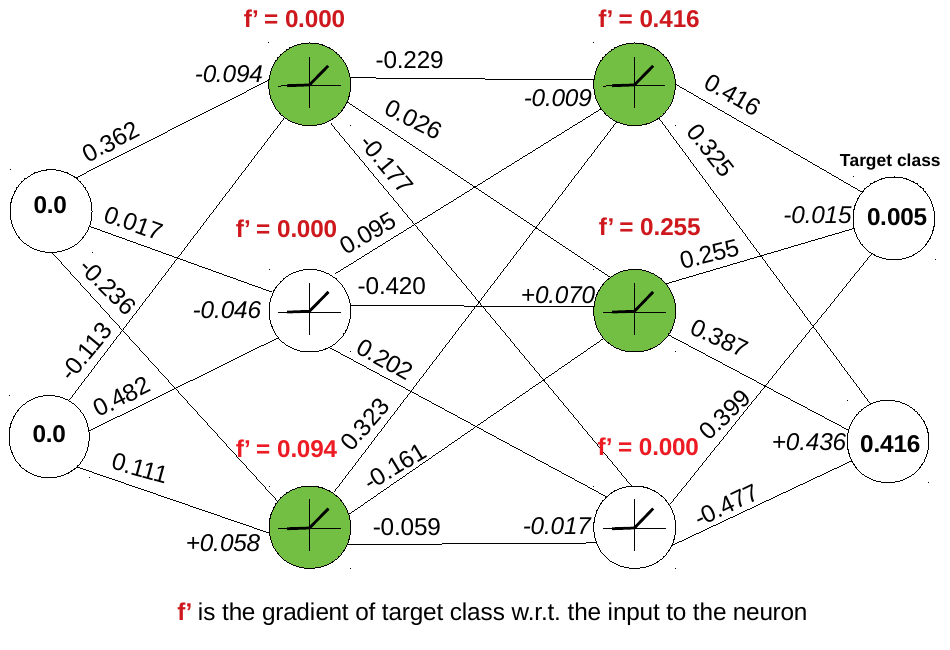}
		\label{fig:mnist-a}
	\end{subfigure}
	\begin{subfigure}{.35\textwidth}
		\includegraphics[width=\textwidth]{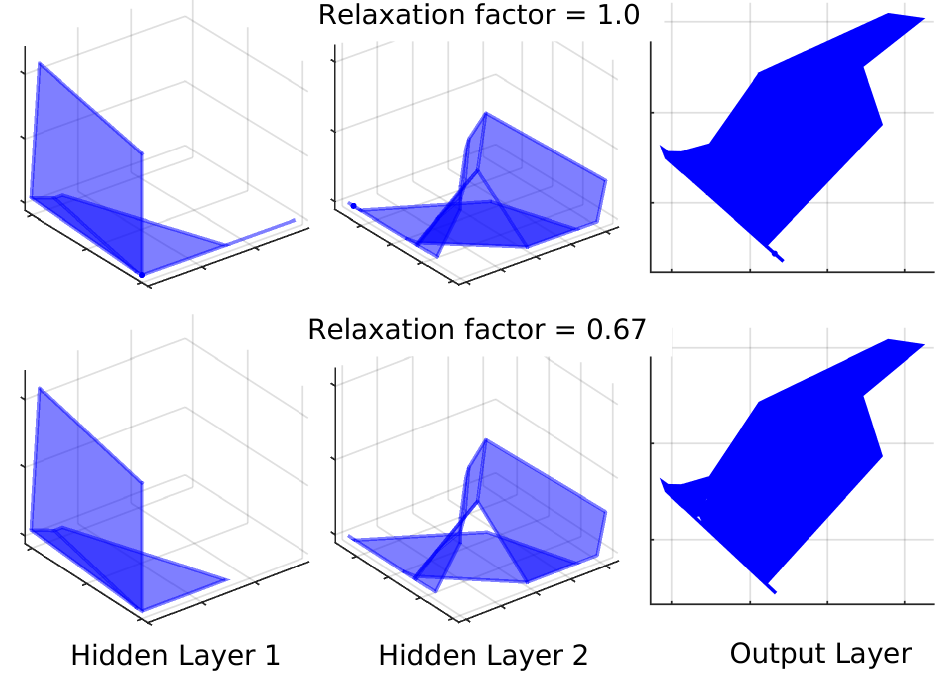}
		\label{fig:mnist-b}
	\end{subfigure}
	\caption{Reachable domains with different factors}
	\label{fig:r3}
\end{figure}

\begin{figure}[ht]
    \centering
    \begin{subfigure}[t]{0.23\textwidth}
        \centering
        \includegraphics[width=\textwidth]{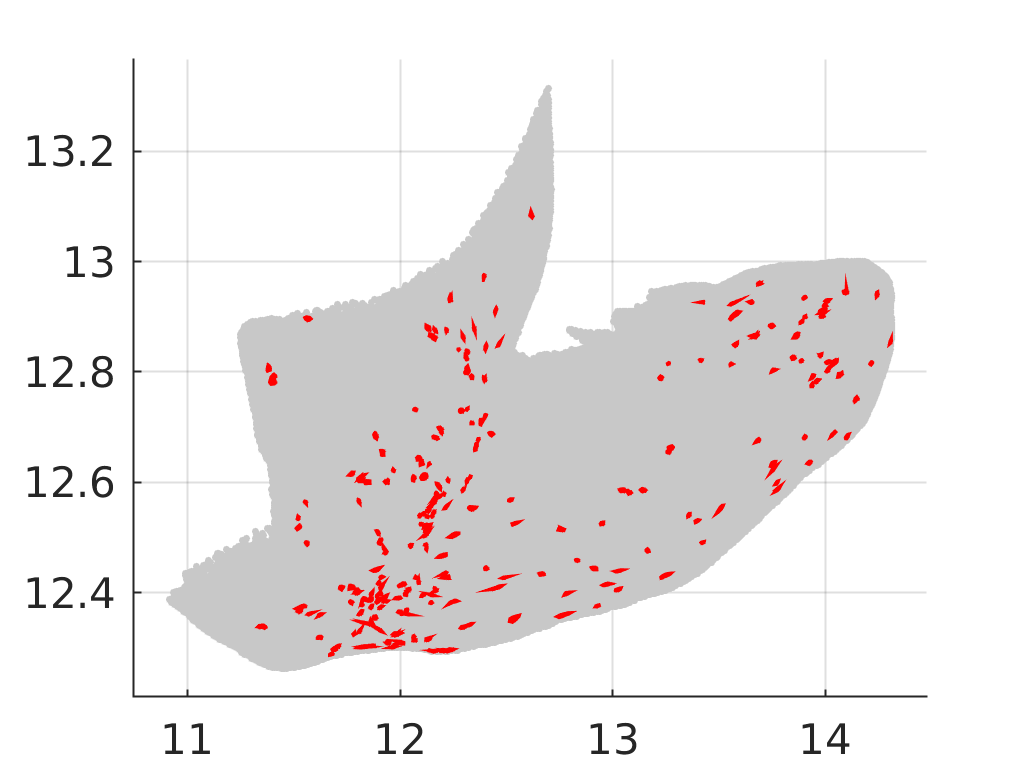}
        \caption{{CIFAR10} factor=0.01}
    \end{subfigure}
     \begin{subfigure}[t]{0.23\textwidth}
        \centering
        \includegraphics[width=\textwidth]{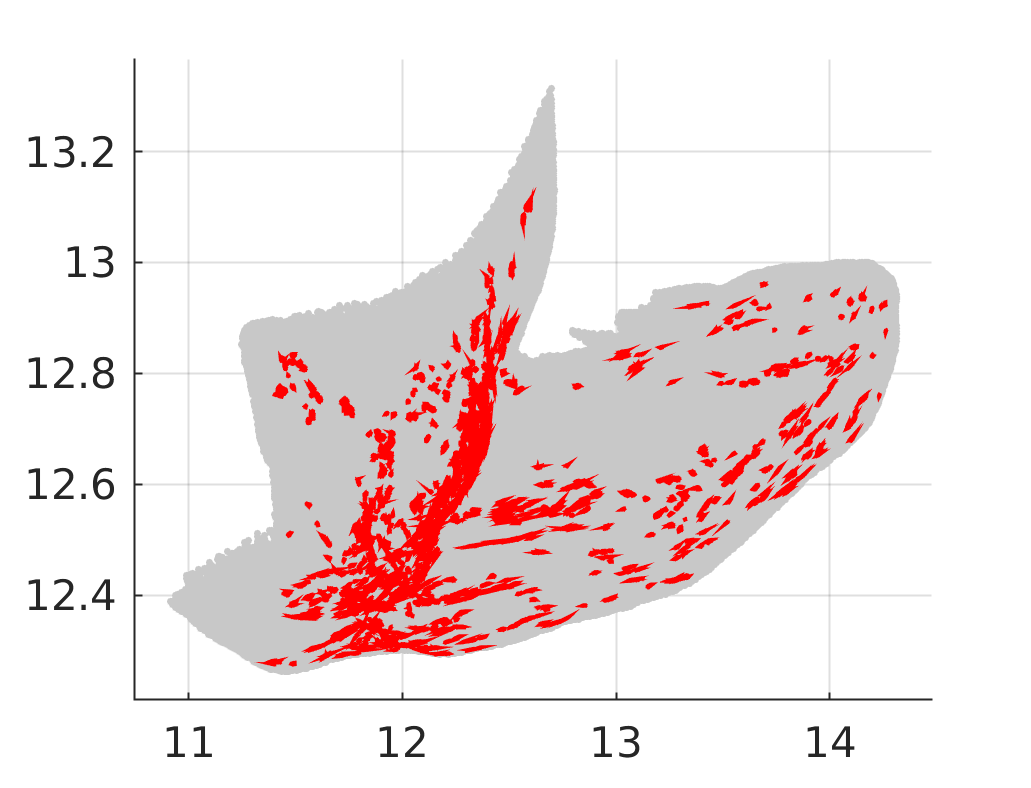}
       \caption{{CIFAR10} factor=0.1}
    \end{subfigure}
      \begin{subfigure}[t]{0.23\textwidth}
        \centering
        \includegraphics[width=\textwidth]{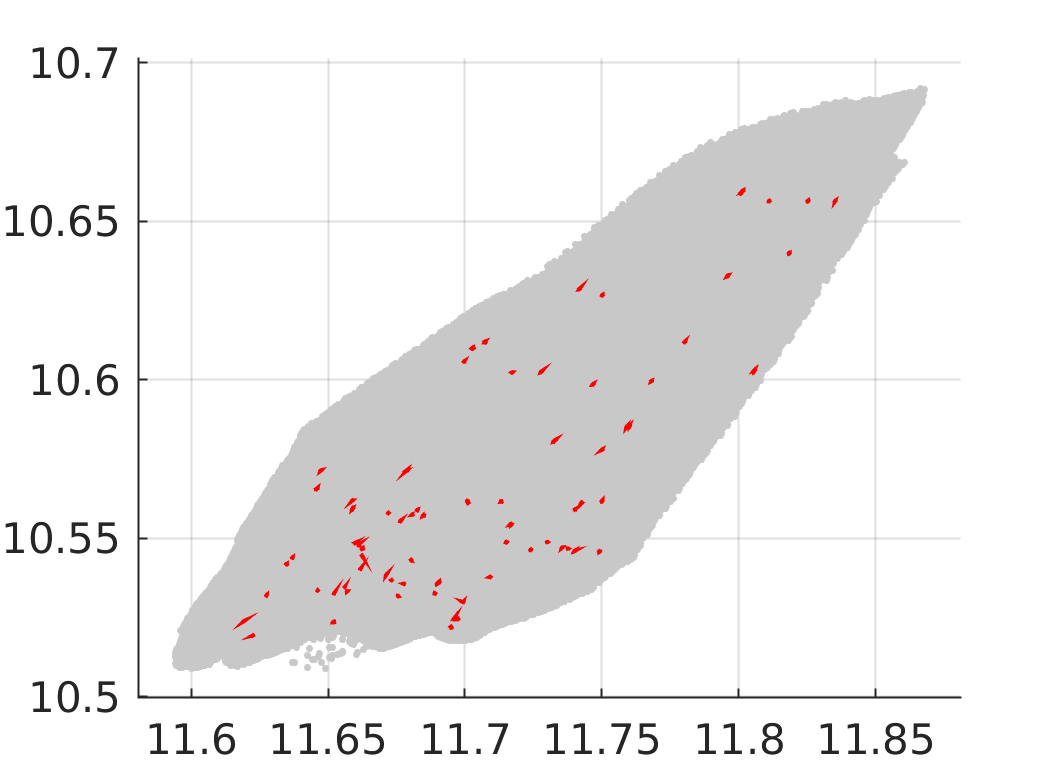}
        \caption{{VGG16} factor=0.01}
    \end{subfigure}
     \begin{subfigure}[t]{0.23\textwidth}
        \centering
        \includegraphics[width=\textwidth]{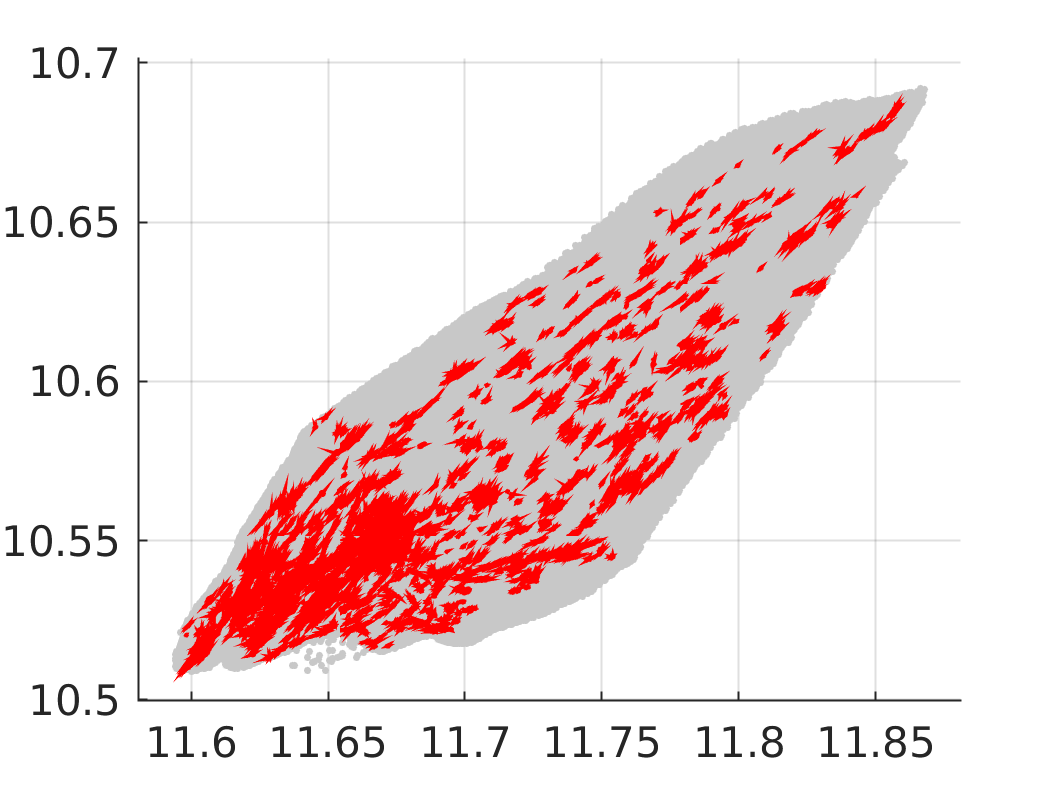}
        \caption{{VGG16} factor=0.1}
    \end{subfigure}
    \caption{The fast computation of reachable sets with different \textit{relaxation} factors. The grey area is the exact reachable domain. The red area are the reachable domains computed by the fast method.The $x$ axis is the logit of the correct class while the $y$ axis is the second highest logit.}
    \label{fig:fastreach1}
\end{figure}

\begin{figure*}[h]
    \centering
     \begin{subfigure}[t]{0.22\textwidth}
        \centering
        \includegraphics[scale=0.55]{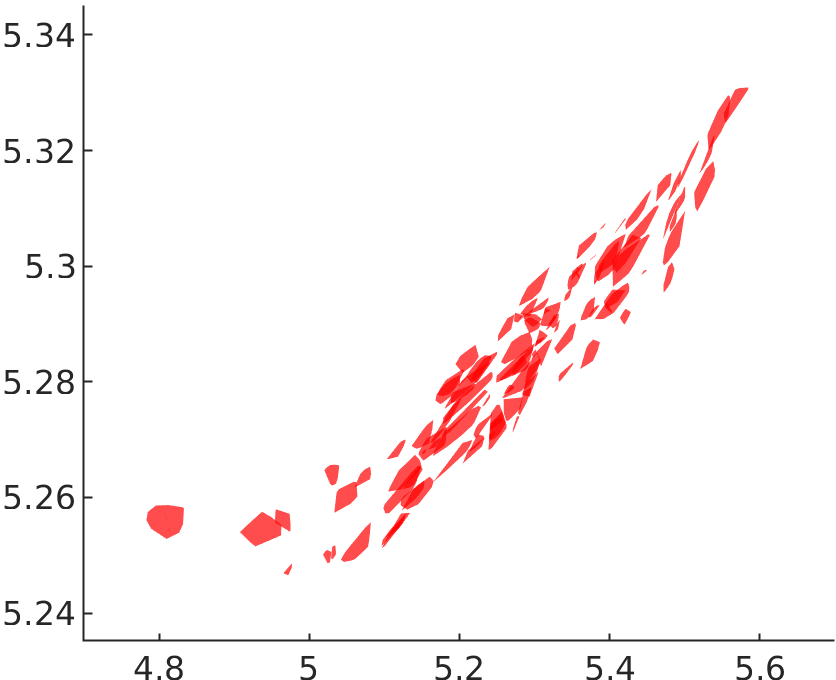}
    \end{subfigure}
    ~
     \begin{subfigure}[t]{0.22\textwidth}
        \centering
        \includegraphics[scale=0.55]{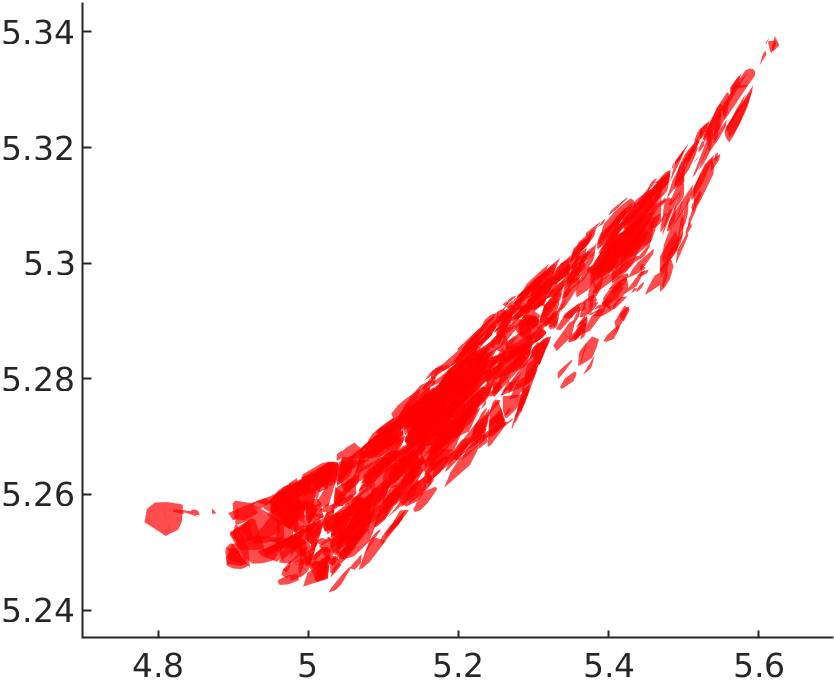}
    \end{subfigure}
    ~
      \begin{subfigure}[t]{0.22\textwidth}
        \centering
        \includegraphics[scale=0.55]{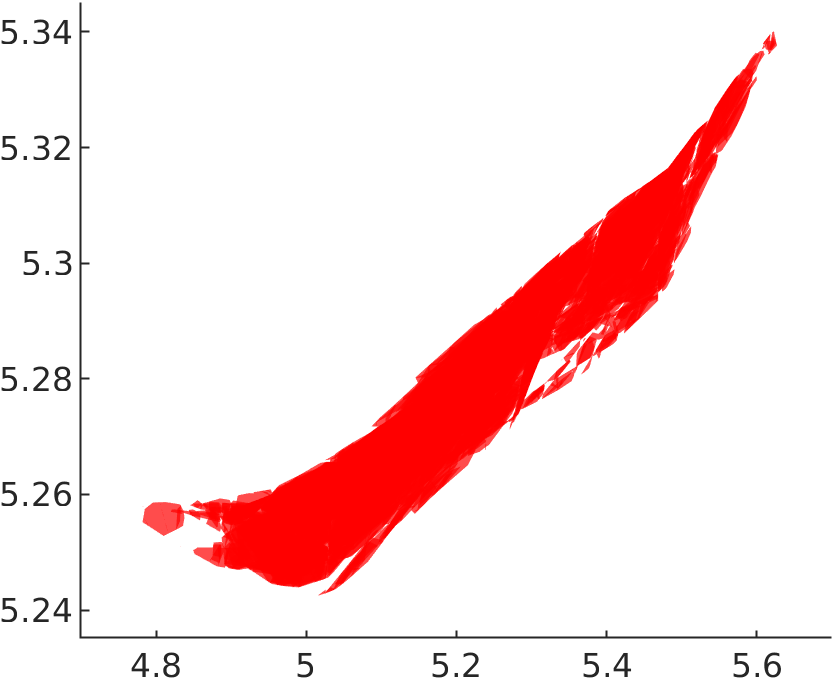}
    \end{subfigure}
    ~
     \begin{subfigure}[t]{0.22\textwidth}
        \centering
        \includegraphics[scale=0.55]{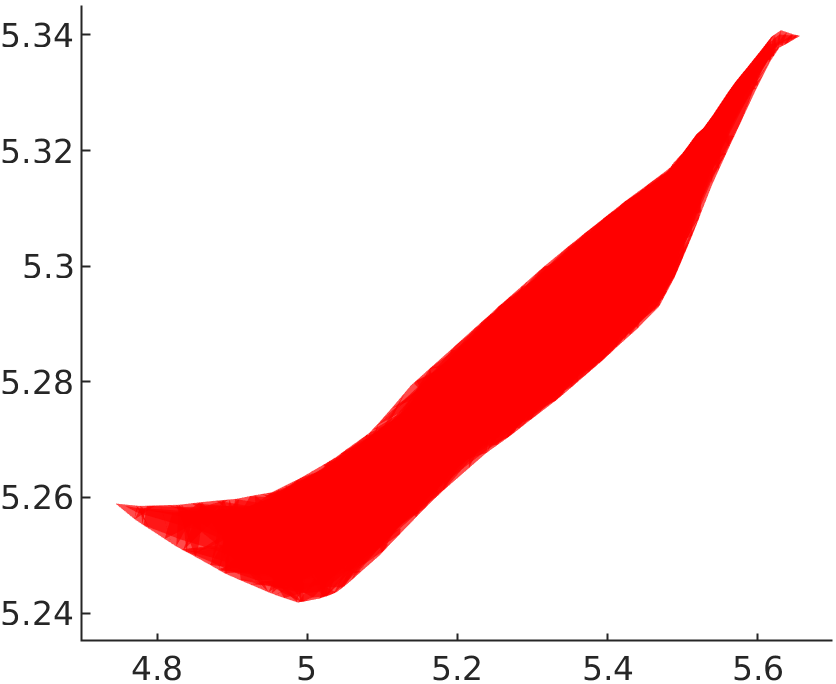}
    \end{subfigure}
    
    \begin{subfigure}[t]{0.22\textwidth}
        \centering
        \includegraphics[scale=0.55]{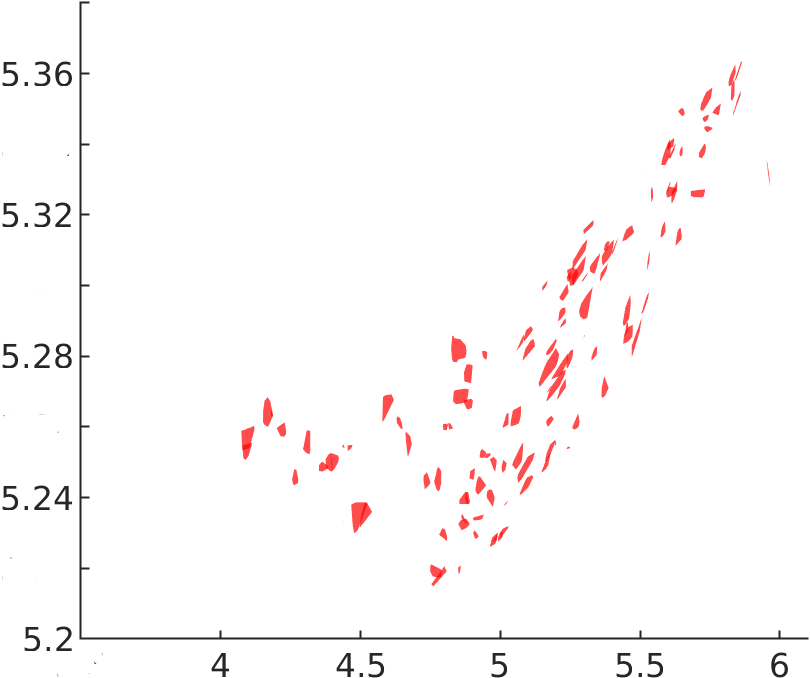}
    \end{subfigure}
    ~
     \begin{subfigure}[t]{0.22\textwidth}
        \centering
        \includegraphics[scale=0.55]{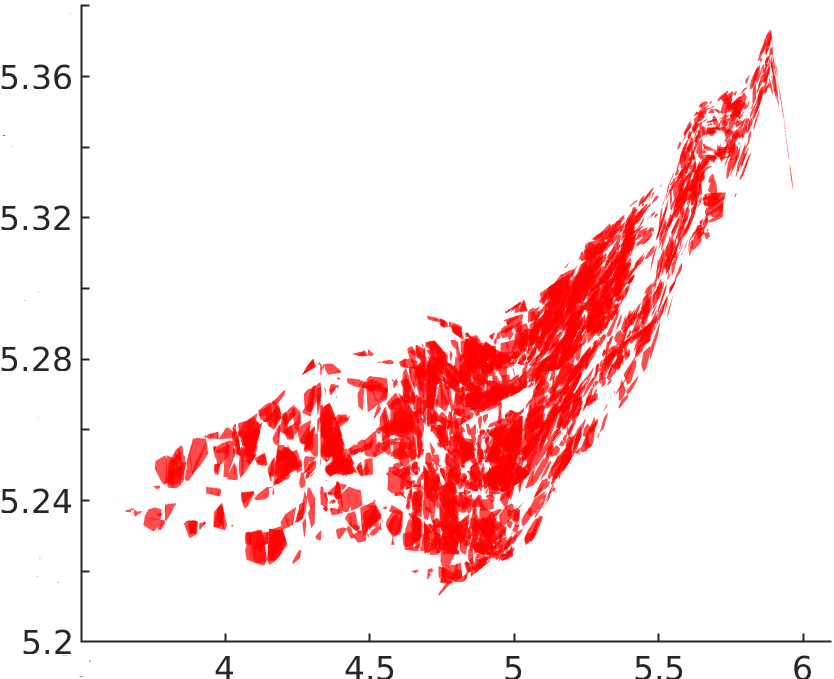}
    \end{subfigure}
    ~
      \begin{subfigure}[t]{0.22\textwidth}
        \centering
        \includegraphics[scale=0.55]{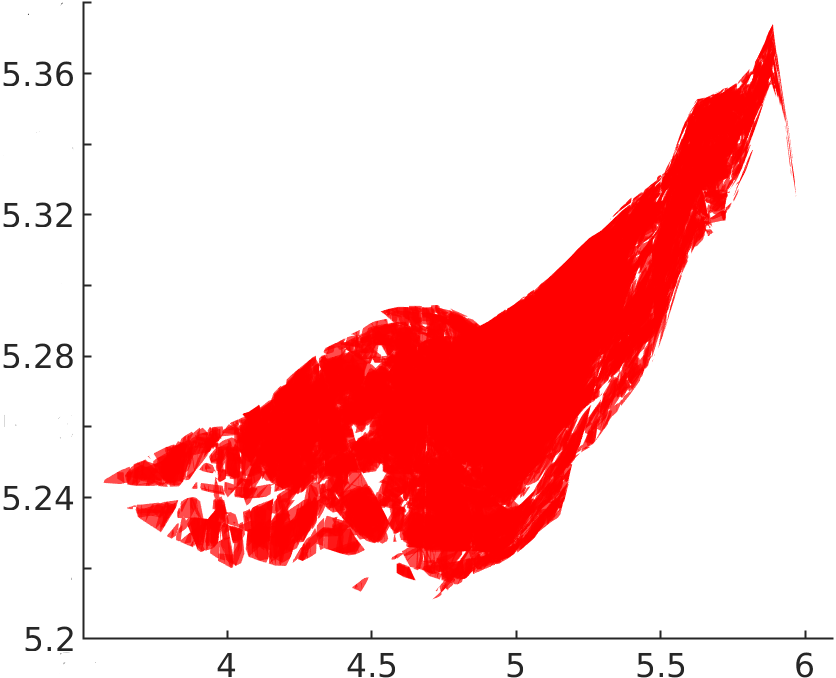}
    \end{subfigure}
    ~
     \begin{subfigure}[t]{0.22\textwidth}
        \centering
        \includegraphics[scale=0.43]{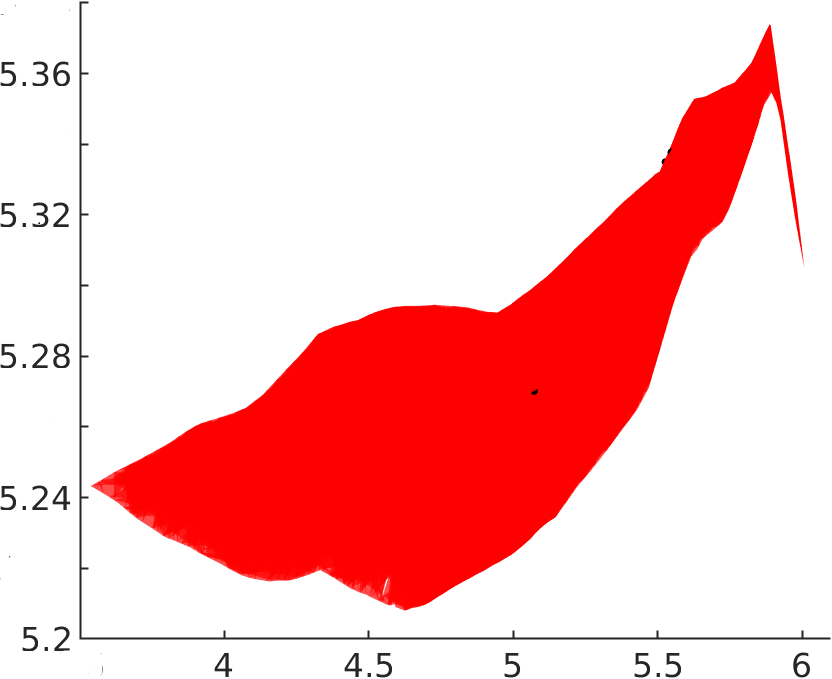}
    \end{subfigure}
    
        \begin{subfigure}[t]{0.22\textwidth}
        \centering
        \includegraphics[scale=0.55]{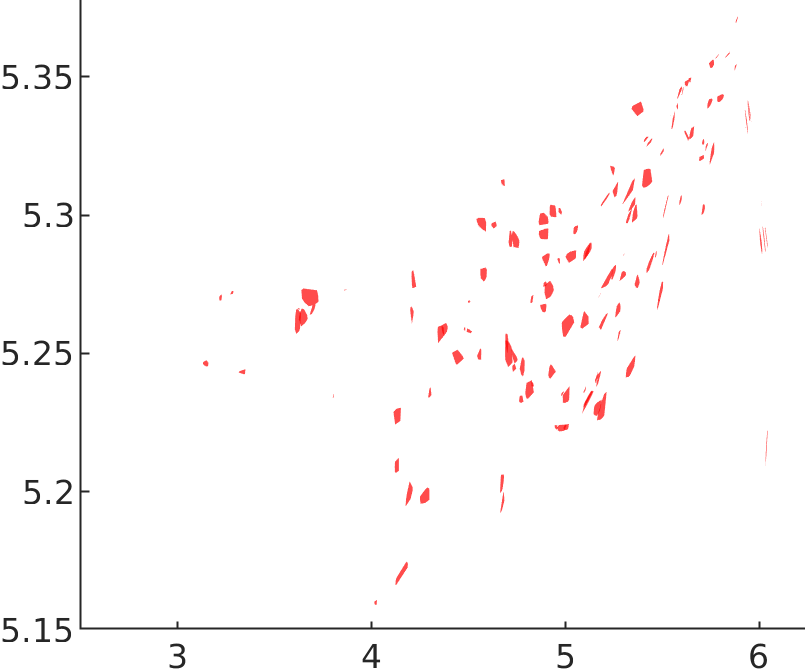}
    \end{subfigure}
    ~
     \begin{subfigure}[t]{0.22\textwidth}
        \centering
        \includegraphics[scale=0.55]{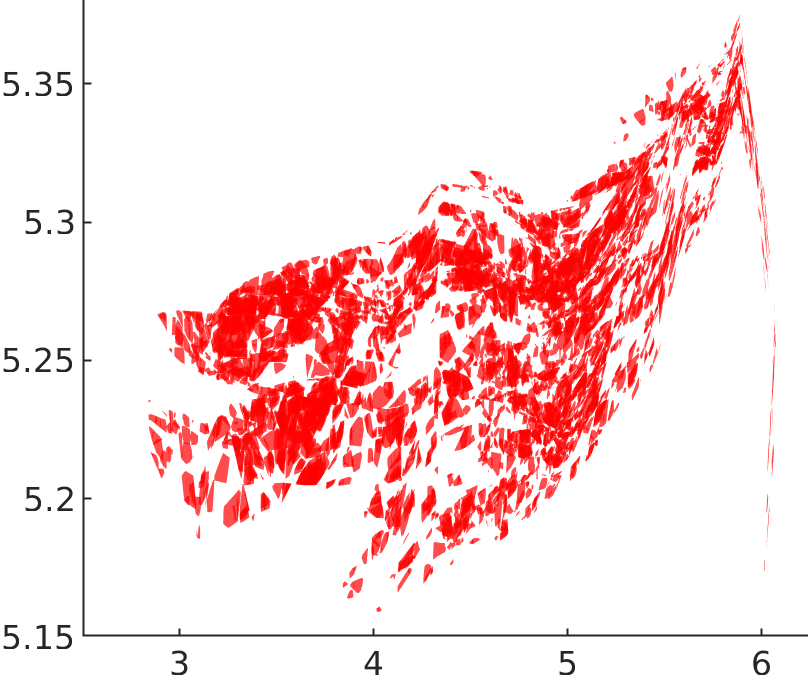}
    \end{subfigure}
    ~
      \begin{subfigure}[t]{0.22\textwidth}
        \centering
        \includegraphics[scale=0.55]{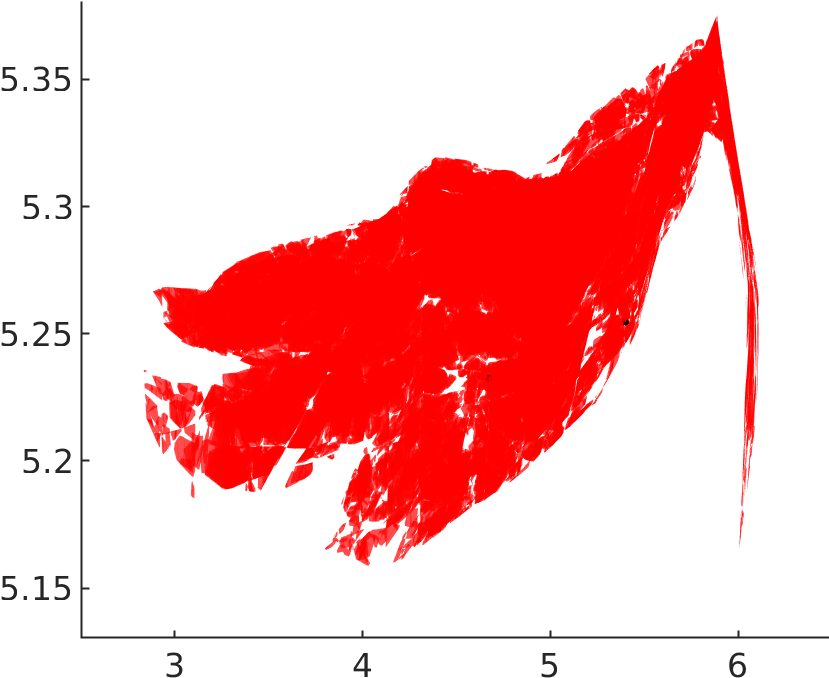}
    \end{subfigure}
    ~
     \begin{subfigure}[t]{0.22\textwidth}
        \centering
        \includegraphics[scale=0.55]{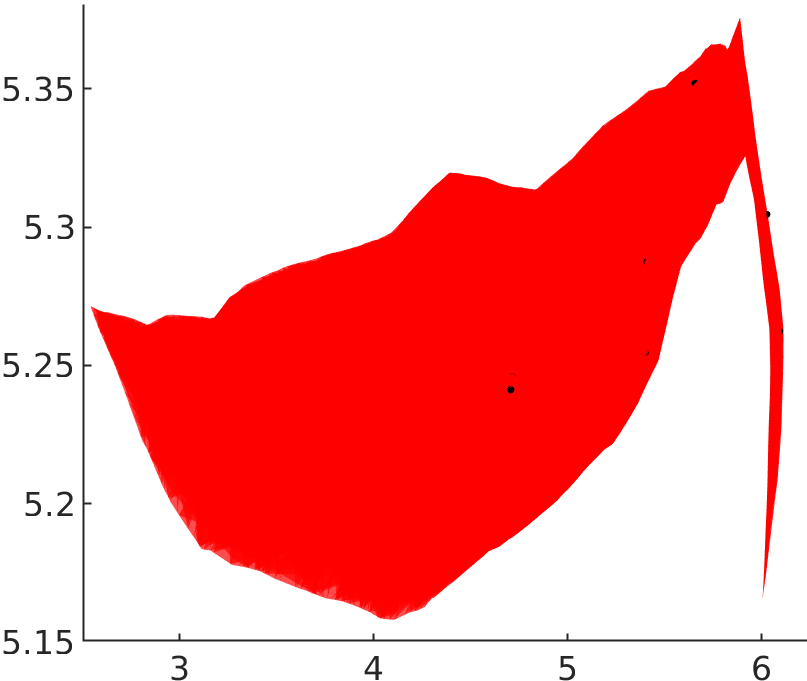}
    \end{subfigure}
    
    \caption{The fast computation of reachable sets with different \textit{relaxation} factors on a CIFAR10 image. The computation in each row is with the different $\epsilon$s. From the top to the bottom, they are respectively $0.02$, $0.06$ and $0.10$. The computation in each column is with different \textit{relaxation} factors. From the right to the left, they are respectively $0.01$, $0.2$, $0.6$ and $1.0$.  The horizontal axis corresponds to the output of the correct class and the vertical axis corresponds to the class which has the most likelihood to be misclassified to. }
    \label{fig:fastreach}
\end{figure*}


\begin{table*}[]
\centering
\renewcommand{\arraystretch}{1.3}
\renewcommand{\tabcolsep}{1.3mm}
\caption{Evaluation of the fast reachability analysis with different settings of epsilons and \textit{relaxation} factors. The result format is \textbf{Running Time}(\textit{sec})/\textbf{ Number of Reachable Sets} for each computation.}
\resizebox{1.7\columnwidth}{!}{\begin{tabular}{ccccccc}\toprule
\textit{relaxation}     & 0.01 & 0.2 & 0.4 & 0.6 & 0.8 & 1.0 \\ \hline
$\epsilon$=0.02 &   $5.3\ /\ 124$    & $8.6\ /\ 1602$     & $12.5\ /\ 4109$     & $23.7\ /\ 14274$     &  $32.3\ /\ 22449$   &  $48.7\ /\ 42931$    \\
$\epsilon$=0.06 &  $5.2\ /\ 126$     & $16.3\ /\ 6409$      &   $37.0\ /\ 22902$    &  $115.5\ /\ 95929$    &  $289.5\ /\ 261140$   & $765.5\ /\ 766849$   \\
$\epsilon$=0.10 &  $ 5.2 \ /\ 143 $    &  $ 24.2 \ /\ 11445 $    &  $ 66.4 \ /\ 44986 $   &  $ 268.3 \ /\ 223970 $   &   $ 968.7 \ /\ 835045 $  & $ 2930.1 \ /\ 2722864$  \\ \bottomrule
\end{tabular}
 \label{tab:fasteval}}
\end{table*}

Here is a running example to demonstrate our algorithm in Figure~\ref{fig:r3}. It starts by ranking the neurons in each layer in terms of sensitivity to perturbations. The relaxation factor defines the percentage of top-ranked neurons considered in the computation of reachable sets. The lower the factor, the fewer the number of neurons considered and the faster the method. The input is x=0 and the perturbed range is $x\pm0.5$. When factor=1.0, it considers all neurons and compute exact reachable domain. The result is shown in the first row where the exact reachable domains of each layer are included. When factor=0.67, it considers 2 neurons(in green) in each hidden layer which have highest absolute gradient of the target output. The algorithm computes the subset of exact reachable domain as shown in the second row.

\section{Experimental Results}
This section presents experimental results illustrating the theoretical
results presented in this paper. (1) We evaluate our reachable set computation with various \textit{relaxation} factors to demonstrate the trade-off between complexity and completeness of the reachable set. (2) We compare the performance of our reachability analysis method to other state-of-the-art tools for verification of CNNs. (3) We apply our method to evaluate the robustness of CNNs trained with various algorithms to pixelwise perturbation. (4) We apply our fast analysis method for falsification of neural networks. 

The fast-analysis method is evaluated with a CNN for the CIFAR10 dataset and the VGG16 for the ImageNet dataset. Here, the input set is created with a perturbation on the independent channels of target pixels. Let $C=[c_1,c_2,\dots,c_n]$ denote the channels of the target pixels and $\epsilon$ be the perturbation range. For each channel $c_{i,i\in \{1,2,\dots,n\}}$, a range $[c_i-\epsilon, c_i+\epsilon]$ is created. Thus, we can obtain a $n$-hyperbox set with a face lattice structure which reflects a $L_{\infty}$ norm bounded domain. In this evaluation, we target one pixel with $\epsilon\text{ = }1.0$. The input set is partitioned into several subsets for parallel computation.
The reachable sets are shown in Fig.~\ref{fig:fastreach1}.  
The density and the number of reachable sets increases with the \textit{relaxation} factor, approaching the exact reachable domain as the relaxation factor increases.
We can notice that with \textit{relaxation} = 0.01, the sparse reachable sets can approximate well the outline of their exact reachable domains. 
{Additional experiments are also included to reveal the running time and the number of reachable sets computed increase along with the \textit{relaxation} factor. The fast-analysis method is evaluated by computing reachable sets w.r.t. one pixelwise variation on a CIFAR10 image. Different \textit{relaxation} factors as well as $\epsilon$s which specifies the variation are considered. The computed reachable sets are presented in Figure~\ref{fig:fastreach}, and their number and the running time is shown in Table~\ref{tab:fasteval}. We can notice that our algorithm can compute ~1000 output sets per second, and that the running time and number of output sets increase exponentially with the \textit{relaxation} factor.}

\begin{table*}[ht]
\caption{Summary of the experimental results. The symbol \textbf{SF} stands for SAFE which indicates no misclassification in the input set. The \textbf{US} stands for UNSAFE indicating the existence of misclassification in the input set. The \textbf{UK} stands for UNKNOWN indicating a failed verification returned from the algorithm. The \textbf{TT} stands for the TIMEOUT which is set to one hour for each verification. The \textbf{TIME} stands for the total computation time for the verification of all images (in seconds). The incomplete result of NNV-Exact for $\epsilon$=[0.05,0.10,0.15] is due to the out-of-memory (128~GB).}
\centering
\renewcommand{\arraystretch}{1.3}
\renewcommand{\tabcolsep}{1.3mm}
\resizebox{1.9\columnwidth}{!}{\begin{tabular}{c|rrrrr|rrrrr|rrrrr|rrrrr}
\hline
\multicolumn{1}{c}{$\epsilon$}      & \multicolumn{5}{c}{0.01}     & \multicolumn{5}{c}{0.05}      & \multicolumn{5}{c}{0.10}       & \multicolumn{5}{c}{0.15}   \\ \hline
\textbf{Methods}      & \textbf{SF} & \textbf{US} & \textbf{UK} & \textbf{TT} & \textbf{TIME} & \textbf{SF} & \textbf{US} & \textbf{UK} & \textbf{TT} & \textbf{TIME}  & \textbf{SF} & \textbf{US} & \textbf{UK} & \textbf{TT} & \textbf{TIME}   & \textbf{SF} & \textbf{US} & \textbf{UK} & \textbf{TT} & \textbf{TIME}           \\ 
NNV-Appr     & 66    & 0    & 34    & 0     & 2,303      & 0   & 0   & 98  & 2    & 29,738 & 0   & 0   & 67  & 33   & 193,907 & 0   & 0   & \textless{}67 & \textgreater{}33 & \textless{}193,907      \\ 
NNV-Exact    & 77    &  22   & 0    &  1    &  27,335    & 3    & 5    & 88     & 3     & 131,138         & --    & --    & --    & --     &--       & --    & --    & --     & --     & --                       \\ 
Deepzono    &76      &0     &24     &0      &1,552      & 0   & 0   & 100 & 0    & 2,077  & 0   & 0   & 100 & 0    & 3,079   & 0   & 0   & 100           & 0                & 3,935                   \\ 
Refinezono    & 76    & 0    & 24    &0      & 1,754      &  0   &  0   & 100    &  0    &  2,383     &0     & 0   &100      &0        &3,382     & 0    & 0   & 100      &0            & 4,224   \\
Deeppoly    & 75    & 0    & 25     & 0      & 7,6019     & 5    & 0     & 95    & 0     & 76,720       & 0     &  0   & 100    & 0      & 76,369        & 0    & 0     & 100               &  0                & 75,887    \\
Refinepoly    & 77    & 0    & 23    &  0    & 146,763    & 10    & 0     & 90     & 0      & 186,558       &0     & 0     &100     & 0       & 198,058       & 0     & 0     & 100               & 0                  & 219,365    \\
{\bf Our Method} & 78  & 22  & 0   & 0    & 5,241 & 35  & 65  & 0   & 0    & 7,549  & 13  & 83  & 0   & 4    & 16,135  & 5   & 87  & 0             & 8                & 22,835            \\ \hline
\end{tabular}}
\label{tab:comparison}
\end{table*}

\begin{figure*}[t]
    \centering
    \begin{subfigure}[b]{0.24\textwidth}
        \centering
        \includegraphics[width=\textwidth]{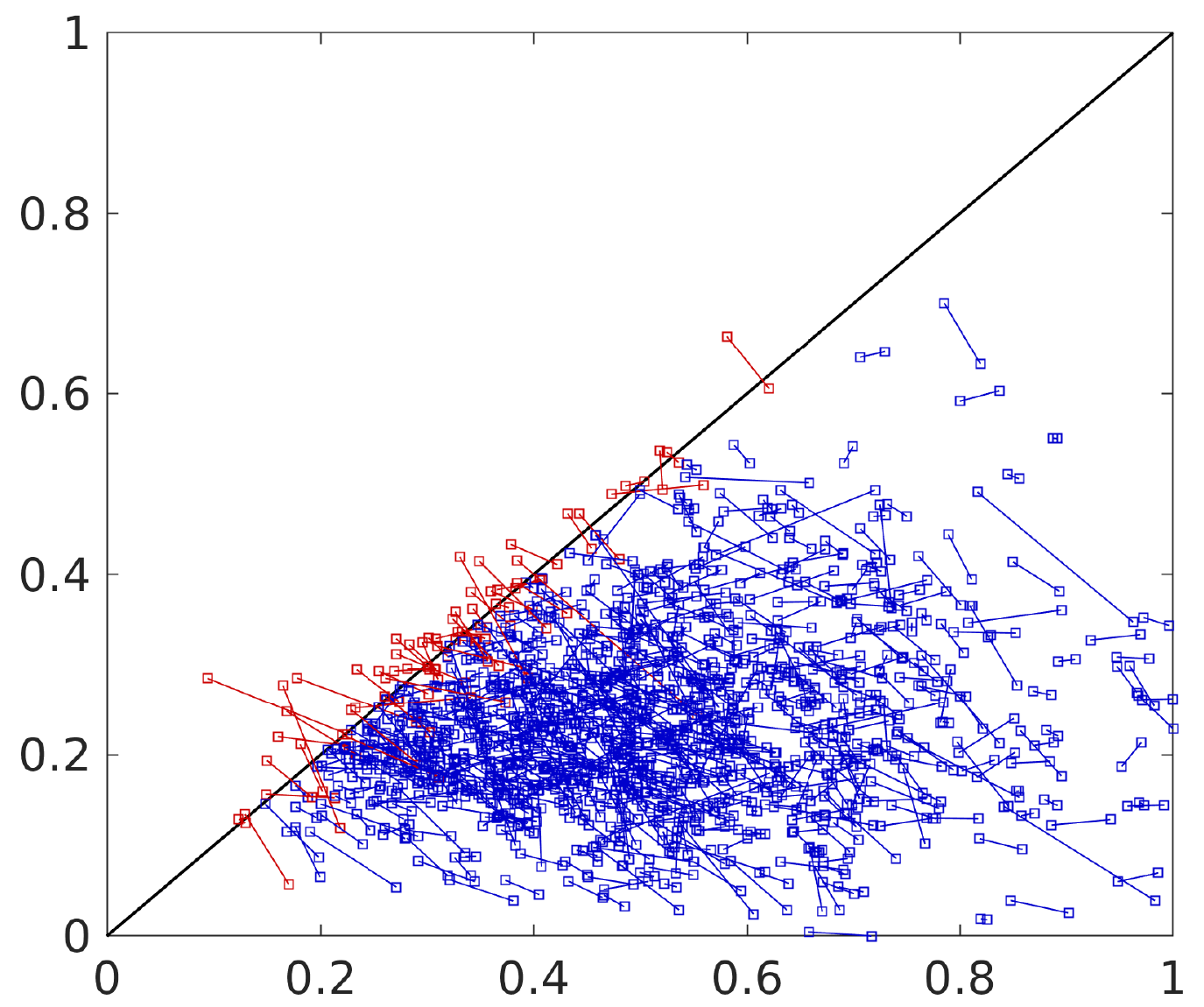}
        \caption{Without Robust Training}
    \end{subfigure}
     \begin{subfigure}[b]{0.24\textwidth}
        \centering
        \includegraphics[width=\textwidth]{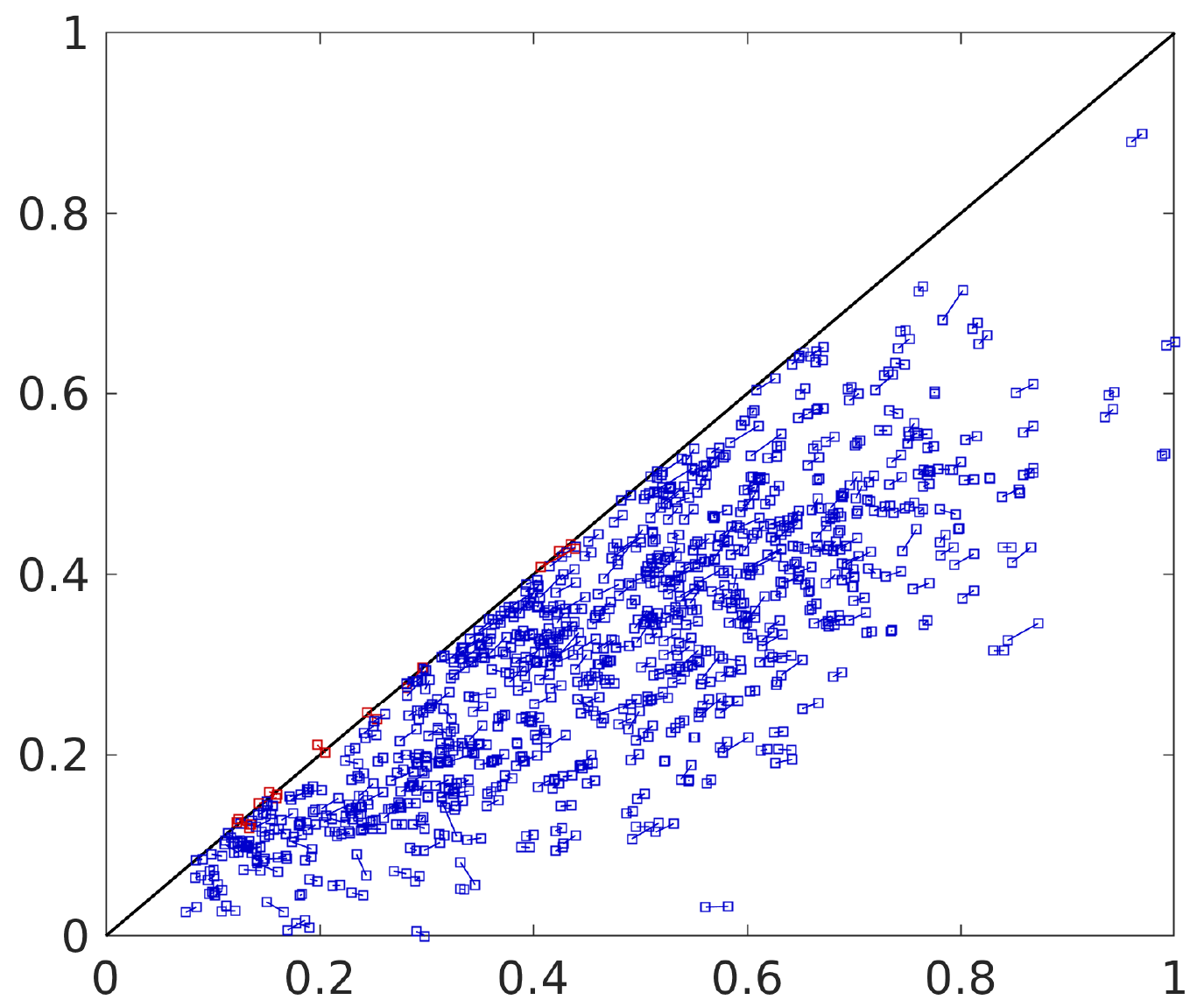}
       \caption{PGD Adversarial}
    \end{subfigure}
     \begin{subfigure}[b]{0.24\textwidth}
        \centering
        \includegraphics[width=\textwidth]{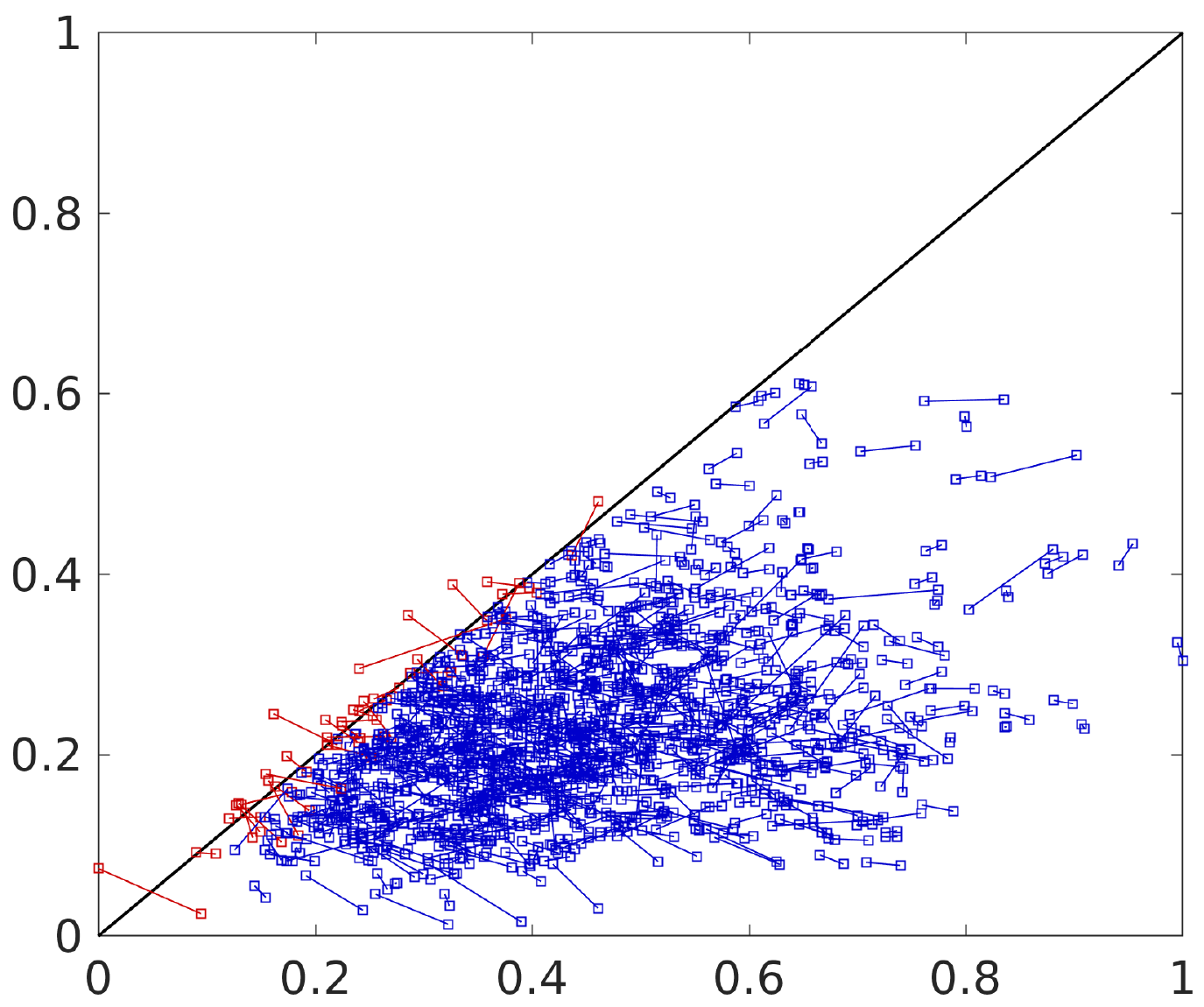}
        \caption{Distillation}
    \end{subfigure}
     \begin{subfigure}[b]{0.24\textwidth}
        \centering
        \includegraphics[width=\textwidth]{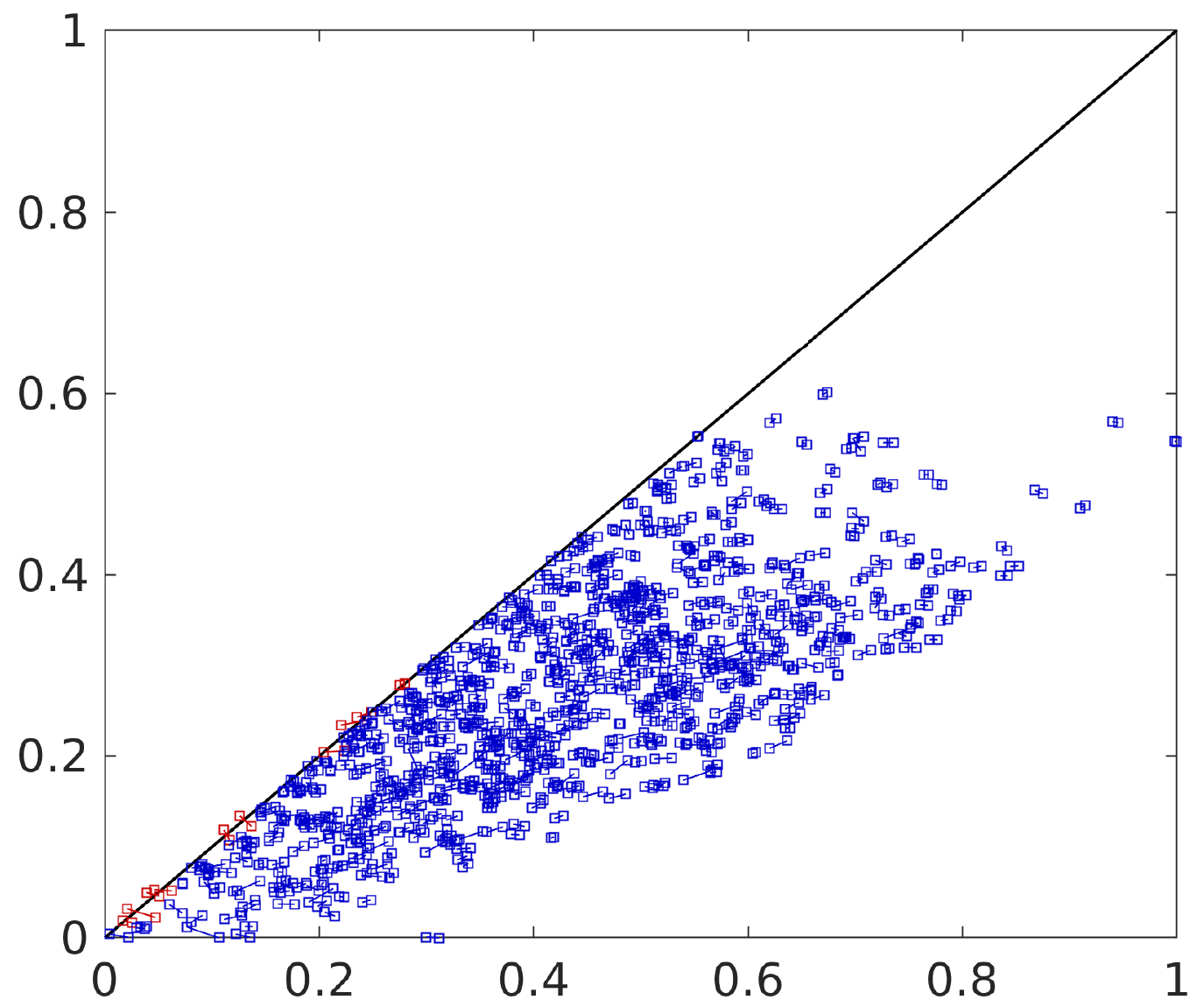}
        \caption{TRADES}
    \end{subfigure}
    \caption{Network robustness evaluation with pixelwise perturbation. The $x$-axis is the output logits of the class \textit{Plane}. The $y$-axis is the second highest logit of the next likely class. The diagonal denotes the decision boundary. The averaged segment lengths of each networks are respectively 0.040, 0.012, 0.034, 0.010. The standard deviations are 0.031, 0.008, 0.025, 0.006. 
    }
    \label{fig:robustness}
\end{figure*}

We evaluate the exact reachability analysis of our approach by comparing it with related state-of-art-methods. Specifically:  Deepzono~\cite{singh2018fast}, Refinezono~\cite{singh2018boosting}, Deeppoly~\cite{singh2019abstract}, Refinepoly~\cite{singh2019beyond} and the NNV tool~\cite{tran2020nnv}. The details of each method are summarized in the supplementary material.
The proposed approach is evaluated on the CIFAR10 network. 
100 images are selected for the test. We apply a perturbation with different $\epsilon$ values to one sensitive  pixel and aim to verify the network. The result is shown in Table~\ref{tab:comparison}. Compared to NNV-Exact, which is also computing the exact reachable sets, our method can quickly verify all images within the timeout limit while the NNV-Exact can only verify the images with $\epsilon=0.01$. Larger $\epsilon$ values will lead to out-of-memory issues. For other over-approximation methods, some of them are faster in running time but mostly fail to verify due to the large conservativeness of the approximated reachable domain. 


We apply our method to analyze the robustness of the CIFAR10 networks trained with different adversarial defense methods. The methods considered are PGD adversarial training~\cite{madry2018towards}, Distillation~\cite{papernot2016distillation} and TRADES~\cite{zhang2019theoretically}. The test images include all the images correctly classified as \textit{Plane} in the test data. For each image, we select the most sensitive pixel and apply reachability analysis to compute the largest output changes w.r.t. a perturbation of $\epsilon$ = 1. 
A reachable set with a large volume indicates that a perturbation can have a large effect on the output classification of the network. On the other hand, a small reachable set volume indicates that the network is robust under perturbation.
This enables us to evaluate the networks' robustness. 
The experimental results over the adversarial training methods are shown in Figure~\ref{fig:robustness}. The figure shows the changes of the output logits before and after the perturbation. The end of the segment that is closer to the boundary denotes the logits after perturbation while the other end denotes the original logits. The segment length reflects the sensitivity. A red segment indicates that an adversarial example is found.
For comparison, logits are normalized into $[0,1]$. 
We notice that networks trained without robust training or with the Distillation method exhibit much larger output variations, compared to the adversarial trained networks with PGD or TRADES. It indicates that the two networks of Figure~\ref{fig:robustness} (a) and (c) are less resistant to the pixelwise perturbation and thus have lower robustness. This supports the results presented in~\cite{athalye2018obfuscated,NEURIPS2020_11f38f8e} where they show that methods such as Distillation which utilize gradient obfuscation give a false sense of robustness. 
This result indicates that our reachability analysis can be useful for a fair evaluation of robustness.

Finally, to evaluate the falsification capability of the fast analysis method, we test it on a CIAFR10 CNN and the VGG16 network. 
Our falsification method changes the input image pixel by pixel until a misclassification is found. 
For each pixel, a fast-reachability computation is conducted with $\epsilon$ = 1 and \textit{relaxation} = 0.01.
For each iteration, we first compute reachable sets for one pixel of the target image, from which we determine the perturbed image which is closer to the decision boundary. 
This perturbed image will be the input for the next iteration. 
The process is repeated until adversarial examples are found or a timeout is reached. 
For the CIFAR10 and VGG16 networks, we attack 400 and 200 images that are correctly classified. 
The results are shown in Fig.~\ref{fig:cifar10_attacks}. 
With different number of pixels allowed to change in each image attack, we compute the number of misclassified images as well as the running time. 
We can notice that the selected CIFAR10 images can be successfully attacked when only a small numbers of pixels is affected. For VGG16, more pixels need to be changed to cause falsification. This is because a single pixel has a lower impact to the image classification due to the larger image size and a deeper network architecture.
The averaged computation time of a single pixel for the CIFAR10 network is $\sim$2 seconds, while for a larger and deeper architecture VGG16 the computation time is $\sim$25 seconds.  

\begin{figure}
    \centering
    \includegraphics[scale=0.36]{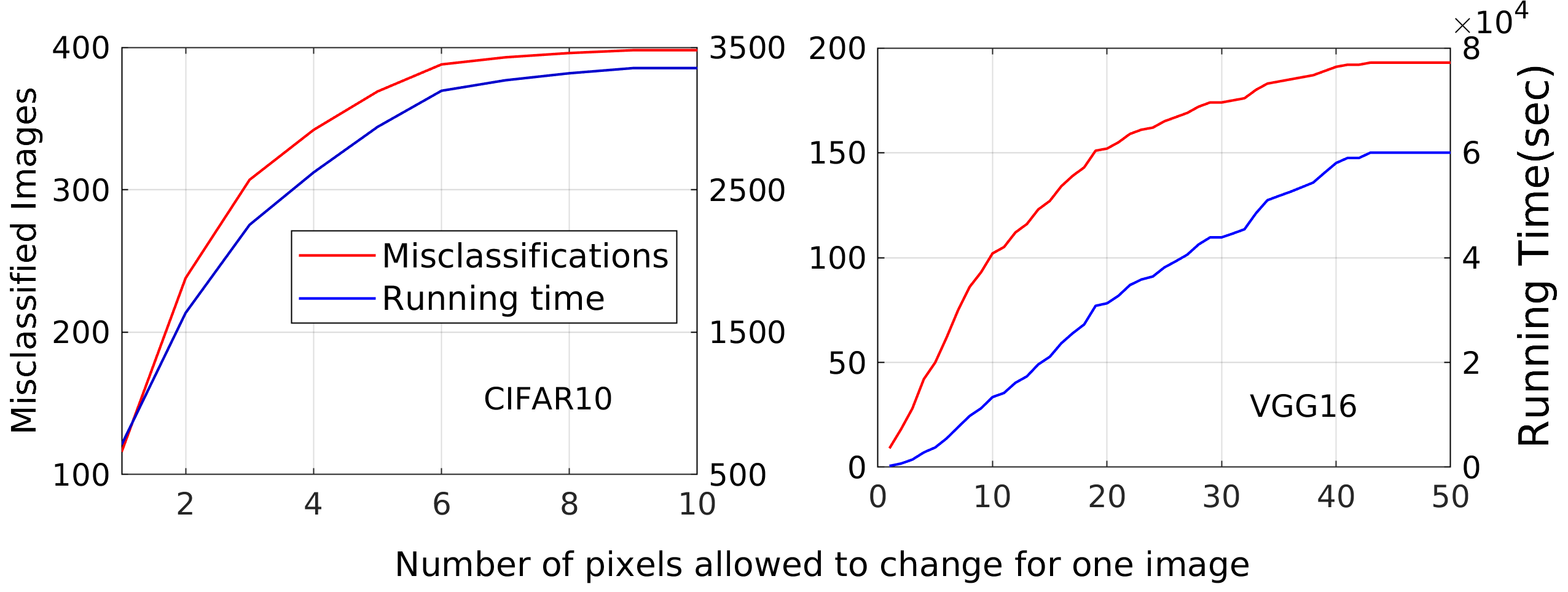}
    \caption{Falsification of CIFAR10 and VGG16 CNNs}
    \label{fig:cifar10_attacks}
\end{figure}

\section{Conclusion}
In this paper, a new approach for computing the reachable sets of CNNs was proposed. 
Our face lattice method could provide an under approximate reachable set in neural network that was faster than complete reachability analysis and wider than usual numerical execution.
We showed the method was applicable for adversarial training as it computes adversarial example domains, and improved the robustness of a vision CNNs efficiently.
We expect the proposed white-box fusion method can disclose neural network mechanism from V\&V perspective since that works black-box nature.
For the future work, we will utilize the proposed method to analyze the semantic segmentation networks for pixelwise classification.

{\small
\bibliographystyle{ieee_fullname}
\bibliography{bibfile}
}

\end{document}